\documentclass{article} % For LaTeX2e
\usepackage[preprint]{colm2026_conference}

\usepackage{microtype}
\usepackage{hyperref}
\usepackage{url}
\usepackage{booktabs}

\usepackage{lineno}

\usepackage{graphicx}
\usepackage{amsmath}
\usepackage{xcolor}
\usepackage{enumitem}
\usepackage{listings}
\usepackage[most]{tcolorbox}
\usepackage{amsthm}
\usepackage{amssymb}
\usepackage{tikz}
\usepackage{subcaption}
\usepackage{lmodern} % Ensures high-quality font rendering
\usepackage{wrapfig}
\usepackage{fancyvrb}
\usepackage{fvextra}
\usepackage[T1]{fontenc}
\usepackage{float}
\usepackage{thmtools}
\usepackage{thm-restate}
\usetikzlibrary{arrows.meta, calc, decorations.markings, positioning}

\newtheorem{definition}{Definition}
\newtheorem{lemma}{Lemma}

\newcommand{\NC}[1]{\mathbf{NC}^{#1}}

\newtcolorbox{promptbox}[1]{
    enhanced,
    sharp corners,
    boxrule=0.5pt,
    colback=white,
    colframe=gray!30,     % Light border
    colbacktitle=gray!10, % Light gray header background
    coltitle=black,       % Black text for title
    fonttitle=\sffamily\bfseries\footnotesize,
    title=#1,
    titlerule=0.5pt,      % Line between title and content
    top=0pt,
    bottom=-8pt,          % Cancels listing margin
    left=4pt,
    right=4pt,
}

\definecolor{darkblue}{rgb}{0, 0, 0.5}
\hypersetup{colorlinks=true, citecolor=darkblue, linkcolor=darkblue, urlcolor=darkblue}

\title{Can Vision-Language Models Solve the Shell Game?}

\author{Tiedong Liu \\
National University of Singapore\\
\texttt{tiedong.liu@u.nus.edu} \\
\And
Wee Sun Lee \\
National University of Singapore\\
\texttt{dcsleews@nus.edu.sg} \\
}

\begin{document}

\ifcolmsubmission
\linenumbers
\fi

\maketitle

\begin{abstract}
Visual entity tracking is an innate cognitive ability in humans, yet it remains a critical bottleneck for Vision-Language Models (VLMs). This deficit is often obscured in existing video benchmarks by visual shortcuts. We introduce VET-Bench, a synthetic diagnostic testbed featuring visually identical objects that necessitate tracking exclusively through spatiotemporal continuity. Our experiments reveal that current state-of-the-art VLMs perform at or near chance level on VET-Bench, exposing a fundamental limitation: an over-reliance on static frame-level features and a failure to maintain entity representations over time. We provide a theoretical analysis drawing connections to the state-tracking problem, proving that fixed-depth transformer-based VLMs are fundamentally limited in tracking indistinguishable objects without intermediate supervision due to expressivity constraints. To address this, we propose Spatiotemporal Grounded Chain-of-Thought (SGCoT): generating object trajectories as explicit intermediate states. Leveraging Molmo2's object tracking ability, we elicit SGCoT reasoning by fine-tuning on synthesized text-only data for alignment. Our method achieves state-of-the-art accuracy exceeding 90\% on VET-Bench, demonstrating that VLMs can reliably solve the video shell-game task end-to-end without external tools. Our code and data are available at \url{https://vetbench.github.io}.
\end{abstract}
% We hope this work facilitates future research on VLMs that leverage SGCoT for video perception and reasoning.

\section{Introduction}

Vision-Language Models (VLMs) have demonstrated remarkable proficiency in general video understanding and reasoning \citep{wu2024longvideobench, wang2025lvbench, hu2025video, fu2025video}. However, their low-level perception, specifically the ability to track entities over time, remains a critical bottleneck. We investigate how video VLMs perform on visual entity tracking tasks like the shell game. While such tasks are often effortless for humans and even some animals \citep{jaakkola2014animals}, they present significant challenges for current VLMs, highlighting a key limitation in fine-grained spatiotemporal perception, a capability essential for many downstream applications, such as embodied AI \citep{thompson2025rem, fung2025embodied} and general game-playing agents \citep{bolton2025sima, magne2026nitrogen}. 

% Existing video benchmarks such as the Perception Test \citep{patraucean2023perception} have attempted to evaluate this capability using real-world recordings of shell game (i.e., cups-game subset). However, our audit reveals that many cups-game clips contain appearance cues: distinctive or transparent cups allow models to solve the task by re-identifying the object from a single frame, rather than performing full temporal tracking across frames. To quantify this, we examine the 3,525 videos in the test split of the Perception Test and identify 189 cups-game clips. After filtering for cases with appearance cues, only 107 video-question pairs remain. Performance drops sharply on this subset. For instance, Gemini-3-Pro \citep{pichai2025new} drops from 80\% on the full dataset to 36.45\%. When restricted to the 3-cup setting and removing no-shuffling instances that test only object permanence, the performance further drops to 30.77\%, which is no better than random guessing ($1/3$). These findings suggest that visual entity tracking is a key limiting factor for model performance on the Perception Test. Easier instances that require little or no tracking can yield high accuracy and thereby inflate the aggregate score, while obscuring failures on the genuinely tracking-dependent cases. Consequently, we argue that the shell-game task exposes a critical failure mode: addressing this task is an important step toward genuine, human-level visual perception.

Existing video benchmarks such as the Perception Test \citep{patraucean2023perception} attempt to evaluate this capability using real-world recordings of the shell game (i.e., cups-game subset). However, our audit reveals that many cups-game clips contain appearance cues: distinctive or transparent cups allow models to solve the task by re-identifying the object from a single frame, rather than performing full temporal tracking across frames. Of the 189 cups-game clips in the 3,525-video test split, 107 video--question pairs remain after filtering out such cues. Performance drops sharply on this filtered subset: Gemini-3-Pro \citep{pichai2025new} drops from 80\% on the full dataset to 36.45\%. When restricted to the 3-cup setting and removing non-shuffling instances that test only object permanence, the performance further drops to 30.77\%---no better than random guessing ($1/3$). These findings suggest that visual entity tracking is a key limiting factor for model performance on the Perception Test. Easier instances that require little or no tracking can yield high accuracy and thereby inflate the aggregate score, while obscuring failures on the genuinely tracking-dependent cases. Consequently, we argue that the shell-game task exposes a critical failure mode: addressing this task is an important step toward genuine, human-level visual perception.

To systematically investigate this limitation, we introduce \textbf{Visual Entity Tracking Benchmark (VET-Bench)}, a synthetic diagnostic testbed designed to isolate spatiotemporal perception from frame-level appearance cues. By using sets of visually identical objects, VET-Bench forces models to track entities solely through motion continuity. Our extensive experiments reveal that current state-of-the-art models such as Gemini-3-Pro perform at or near chance level on VET-Bench. These results align with our audit of the Perception Test, suggesting that models rely heavily on static appearance features rather than genuine entity tracking. When these discriminative appearance cues are removed, the VLMs struggle to maintain coherent entity representations over time.
% , far below human performance

Humans solve the shell game through an intuitive perceptual process that requires little conscious effort. This raises a fundamental question: are transformer-based \citep{Vaswani+2017} VLMs inherently capable of solving the video shell game? By drawing connections to the state-tracking problem \citep{merrill2024illusion}, we provide a theoretical analysis proving that visual entity tracking is \textbf{$\mathbf{NC}^1$-complete}. This suggests that fixed-depth transformers are fundamentally limited in solving general visual entity tracking tasks without intermediate computation due to expressivity constraints \citep{merrill2023parallelism, merrill2024the, feng2023towards, li2024chain, huang2025transformers}. We empirically verify that VLMs struggle to learn the shell game using direct-answer supervision, even with sufficient training.

% The model fails to capture the correlation between the final answer and long sequences of video tokens, mirroring the well-known difficulty of learning parity tasks \citep{wies2023subtask, kim2025transformers} without intermediate supervision or chain-of-thoughts (CoTs) \citep{wei2022chain}.

To address this, we propose \textbf{Spatiotemporal Grounded Chain-of-Thought (SGCoT)}. Leveraging the object tracking capabilities of Molmo2 \citep{clark2026molmo2}, we transform perception into a reasoning process where the model explicitly generates spatiotemporal grounded object trajectories as intermediate states before providing the final answer. By fine-tuning on synthetic text-only data for alignment, we elicit SGCoT in Molmo2 and achieve state-of-the-art accuracy exceeding 90\% on VET-Bench, demonstrating that VLMs can reliably solve the video shell game task end-to-end without relying on external tools.

\begin{figure*}[t]
\centering
\includegraphics[width=\linewidth]{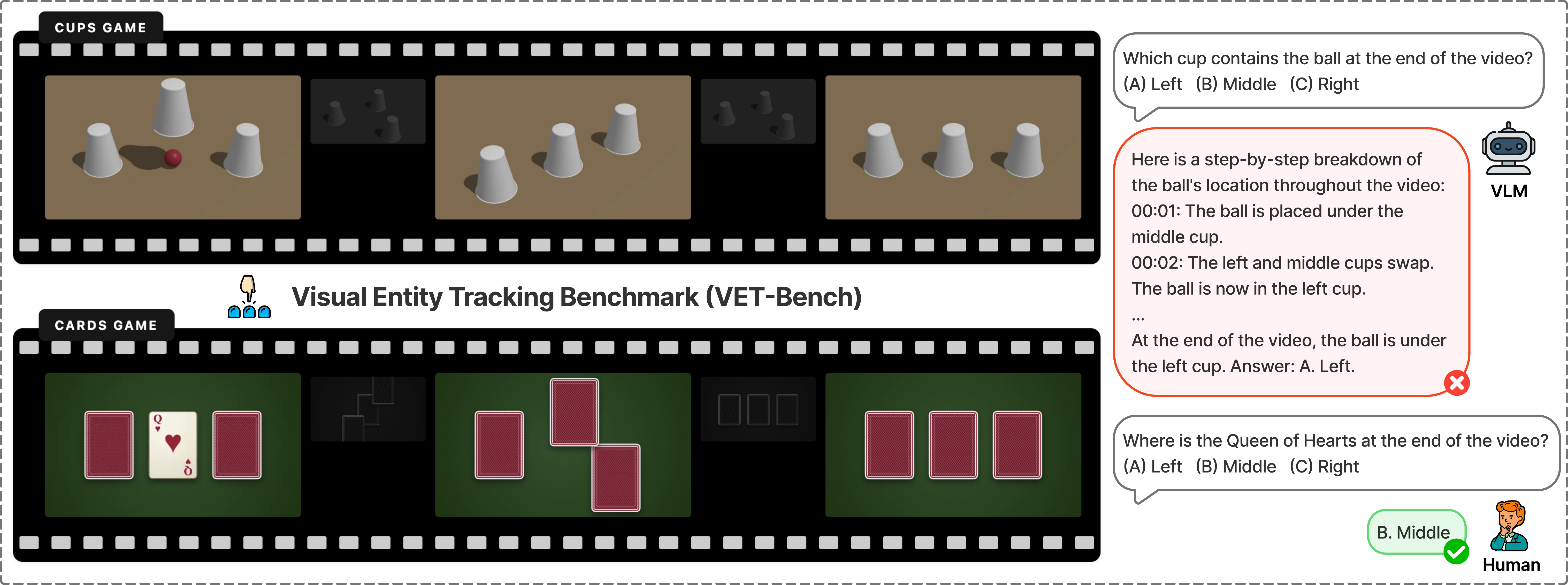}
\caption{Overview of VET-Bench.}
\label{fig:demo}
\end{figure*}

\paragraph{Contributions.} Our work offers the following key contributions:
\begin{itemize}

    \item We identify visual entity tracking as a critical bottleneck in the Perception Test benchmark and introduce VET-Bench, a synthetic diagnostic benchmark featuring visually identical objects, providing a rigorous testbed for evaluating the visual perception of VLMs by eliminating appearance-based shortcuts. 
    \item We conduct a comprehensive evaluation of state-of-the-art proprietary and open-source video VLMs, revealing that all tested models, including frontier models like Gemini-3-Pro, perform near random chance on VET-Bench.
    \item We provide an $\mathbf{NC}^1$-completeness proof for the visual entity tracking task, establishing the necessity of intermediate computation for transformer-based VLM architectures. We empirically verify that training only on direct answers without intermediate supervision fails to learn the shell game even with sufficient training.
    \item We demonstrate that VLMs are capable of reliably solving the shell game end-to-end without external tools through Spatiotemporal Grounded Chain-of-Thought (SGCoT). By transforming visual perception into a reasoning process and aligning Molmo2 to generate object trajectories as CoT, we achieve state-of-the-art performance exceeding 90\% accuracy on VET-Bench. 
\end{itemize}

\section{Data Generation}
\label{sec:data_gen}
\subsection{Task Formulation}
\label{sec:tasks}

We consider a video sequence $\mathcal{V} = \{F_0, \dots, F_T\}$ containing $N$ visually indistinguishable objects, identified by their initial indices $i \in [N] = \{1, \dots, N\}$. The shuffling process induces a global permutation $\pi$, mapping the object at initial index $i$ to its final index $\pi(i)$. Given a target object $i \in [N]$ designated at $t=0$ via a contextual cue $C$, which may take the form of a visual highlight or a linguistic description, the objective is to determine its terminal index $\pi(i)$ in the final frame. 

To ensure the task is well-posed without ambiguity, we enforce a continuity constraint: the maximum displacement $d$ of any object between consecutive frames must satisfy $2d < \Delta$, where $\Delta$ is the minimum spatial separation between any two objects. This constraint mirrors the temporal Nyquist criterion in video sampling, to prevent identity aliasing during object crossovers. In real-world videos, this condition is naturally satisfied by a sufficiently high frame rate relative to object speeds. In VET-Bench, each swap lasts 2 seconds, so the minimum sampling rate required to resolve a swap without ambiguity is 1 FPS (Section~\ref{sec:settings}). 

\subsection{Task Suite}
Similar to CLEVR \citep{johnson2017clevr}, CLEVRER \citep{yi2020clevrer}, and CATER \citep{girdhar2020cater}, our synthetic data generation pipeline provides fine-grained control over environmental parameters, reducing the risk of data leakage and overfitting commonly observed in real-world static benchmarks. Videos are rendered using \texttt{three.js}, supporting full synthetic variation in color, material, texture, lighting, and camera viewpoint. This design enables the generation of unlimited episodes, effectively mitigating the memorization issues inherent in fixed datasets. Moreover, the pipeline allows precise control over simulation parameters such as object count and swap count, enabling diagnostic evaluation of how individual factors influence model performance. We focus on the following two canonical visual entity tracking tasks from the Perception Test (Fig.~\ref{fig:demo}):
\begin{itemize}
    \item \textbf{Cups Game.} Also known as the \textbf{Shell Game}, this task requires tracking a ball hidden beneath visually identical opaque containers that undergo positional swaps. 
    
    \item \textbf{Cards Game.} Modeled after \textbf{Three-Card Monte} (``Find the Queen''), this task requires tracking a card after being flipped face-down and shuffled.
\end{itemize}

% \subsection{Task Features}
Crucially, VET-Bench simulates realistic shell-game videos by ensuring that no single frame reveals either (i) the \textbf{target’s identity} (Section~\ref{sec:compare_with_PT} on the Perception Test) or (ii) the \textbf{swap operation} (Section~\ref{sec:compare_with_VRB} on VideoReasonBench). By removing static, frame-level cues, VET-Bench forces VLMs to rely exclusively on fine-grained spatiotemporal perception.

% ---------- 9. Results ----------
\section{Experiment}
\label{sec:results}
\begin{figure*}
    \centering
    \includegraphics[width=\linewidth]{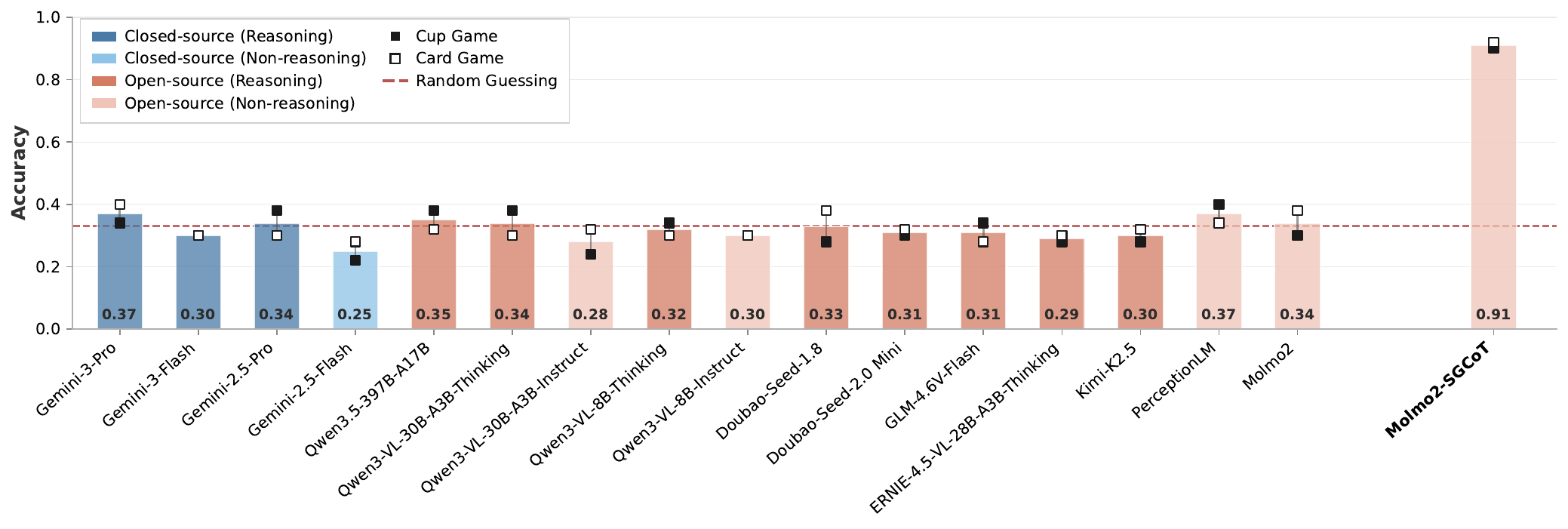}
\caption{Performance on VET-Bench, consisting of 50 cups-game and 50 cards-game videos featuring 3 objects and 5 swaps ($\sim$12 seconds). Existing VLMs all perform near random chance. Molmo2-SGCoT is a fine-tuned model based on Molmo2 that leverages Spatiotemporal Grounded Chain-of-Thought (SGCoT) to solve the shell game (Section~\ref{sec:sgcot}).}
    \label{fig:leaderboard}
\end{figure*}

\subsection{Experimental Setup}

\paragraph{Models.}
We evaluate a comprehensive suite of proprietary and open-source VLMs that natively support video input, including Gemini-3 \citep{pichai2025new}, Gemini-2.5 \citep{comanici2025gemini}, Qwen-3.5 \citep{qwen35blog}, Qwen3-VL \citep{bai2025qwen3}, GLM-4.6V-Flash \citep{vteam2025glm45vglm41vthinkingversatilemultimodal}, Ernie-4.5 \citep{ernie2025technicalreport}, Doubao-Seed-2.0 \citep{bytedance2026seed2,guo2025seed1}, Kimi-K2.5 \citep{team2026kimi}, PerceptionLM \citep{cho2025perceptionlm}, and Molmo2 \citep{clark2026molmo2}. For each model, we use the maximum supported frame rate, or the default sampling rate where unspecified. Reasoning models are evaluated using their default thinking configurations. Detailed prompts are included in Fig.~\ref{fig:gemini-3-pro-cup}.

\paragraph{Metrics.} All tasks are evaluated in the standard multiple-choice question answering (MCQA) format to facilitate automated grading. Performance is measured via Top-1 Accuracy. For a task involving $N$ objects, the random baseline is $1/N$. 

\label{sec:settings}
\paragraph{Settings.} VLMs adopt different frame sampling strategies. To prevent performance discrepancies arising solely from temporal resolution bottlenecks, we standardize each swap operation to a duration of 2.0 seconds. This ensures that models with sparse sampling (e.g., 1 FPS) capture at least 2 frames per swap—the theoretical minimum required to resolve a swap without ambiguity. We evaluate 5 swaps per episode, resulting in a total duration as short as 12 seconds (comprising a 2s initial phase followed by a 10s shuffling phase) to balance temporal resolution and context length constraints.

\subsection{Results}

The results in Fig.~\ref{fig:leaderboard} reveal that all evaluated Vision-Language Models (VLMs) perform near the random guessing baseline. This failure is universal across model sizes and for both reasoning and non-reasoning models. Based on our qualitative analysis of failure modes (Figs.~\ref{fig:gemini-3-pro-cup} to~\ref{fig:model_performances}), we categorize the errors into the following three primary patterns:

% We provide examples of responses from various VLMs in Appendix C. 

\paragraph{Direct Answer} Several models, particularly non-reasoning models such as Molmo2, PerceptionLM, Doubao-Seed-1.8, and GLM-4.6V-Flash, often output only a final answer without any CoTs. Their responses appear to be random guessing.

\paragraph{Coarse Description} Some models successfully identify the initial state but fail to perceive the critical shuffling phase, instead relying on high-level semantic descriptions (Fig.~\ref{fig:model_performances}). For example, \emph{``The cups are shuffled in a shell game-like motion, with the ball remaining under one of the cups throughout the video.''} (Qwen3-VL-8B-Instruct) and \emph{``00:00:02.000 onwards: All three cups begin moving around in a shuffling pattern. 00:00:12.000: The cups return to their original positions.''} (Kimi-K2.5). Such reasoning collapses fine-grained swap events into a coarse global description, making the final prediction effectively random.

\label{sec:hallucination}
\paragraph{Inaccurate Perception and Hallucination} Models with stronger reasoning capabilities, such as Gemini-3-Pro and Gemini-3-Flash, do attempt to generate explicit swap sequences. While the thinking steps are linguistically coherent and logically valid (e.g., \emph{``... Move 1 (00:02 - 00:03): The leftmost cup (with the ball) swaps with the middle cup. The ball is now in the middle position. Move 2 ...''}), they are often grounded in incorrect visual perceptions: models misidentify which entities are moving or hallucinate swaps that never occur (Figs.~\ref{fig:gemini-3-pro-cup} and~\ref{fig:gemini-3-pro-card}). Given that each reasoning step depends on the correctness of the preceding one, any mistake in an intermediate step inevitably leads to an incorrect final prediction. Symbolic reasoning cannot compensate for failures in pixel-level grounding. Consequently, the final answer remains statistically indistinguishable from random guessing.

% \begin{quote}
% ``Move 1 (00:02 - 00:03): The leftmost cup (with the ball) swaps with the middle cup. The ball is now in the middle position. Move 2 ...''
% \end{quote}

\begin{figure*}[t] % [t] pulls it to the top of the page
    \centering
    \begin{subfigure}{0.48\textwidth}
        \centering
        \includegraphics[width=\linewidth]{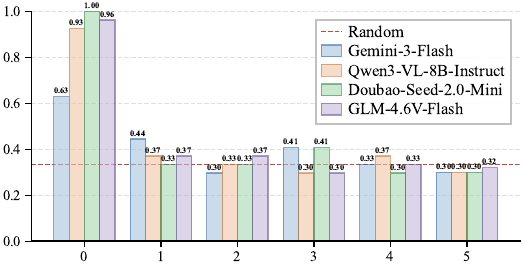}
        \caption{Accuracy vs. Swap Count.}
        \label{fig:swap_count_a}
    \end{subfigure}
    \hfill
    \begin{subfigure}{0.48\textwidth}
        \centering
        \includegraphics[width=\linewidth]{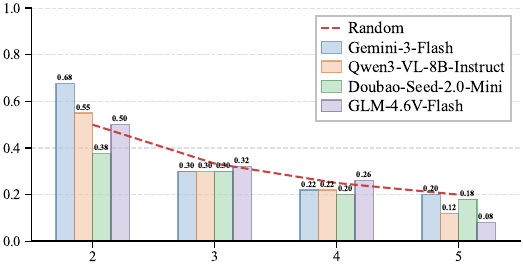}
        \caption{Accuracy vs. Object Count.}
        \label{fig:swap_count_b}
    \end{subfigure}
    \caption{Performance of VLMs under different swap and object counts.}
    \label{fig:combined_swap_count}
\end{figure*}

\subsection{Swap Count}
\label{sec:swap_count}
We further examine how the number of swap operations per episode affects performance (Fig. \ref{fig:swap_count_a}). For zero-swap episodes (which test object permanence only), most models achieve near-perfect accuracy except Gemini-3. Models like Doubao-Seed-2.0-Mini can generate the final answer directly without any CoTs, yet achieve perfect scores. This is expected, as the ball’s location is directly observable in the opening frames, mirroring the static frame-level cues. Surprisingly, performance drops substantially with just \textbf{one swap}, and quickly converges to the random-guessing baseline thereafter.

Interestingly, Gemini-3-Pro and Gemini-3-Flash exhibit a unique failure mode in zero-swap scenarios. With our standard prompt (\emph{``Which cup contains the ball at the end of the video?''}), they often correctly identify the initial state but then hallucinate a long sequence of non-existent swaps, causing the final prediction to be incorrect (Fig.~\ref{fig:gemini-3-flash-hallucination}). This behavior matches the \textbf{event hallucination} described by \cite{zhang2024eventhallusion}. When the prompt is changed to ask for the ball’s location at the start (\emph{``Which cup contains the ball at the start of the video?''}), accuracy returns to near-perfect. We did not observe this prompt-sensitive hallucination pattern in other evaluated VLMs.

\subsection{Object Count}
\label{sec:object_count}
We analyze performance across varying object counts $N=2,3,4$ (Fig.~\ref{fig:swap_count_b}). Notably, even at $N=2$, models fail to significantly outperform the random baseline. In this simplest case, the task reduces to a parity problem, where each swap simply inverts the state (position) of the target object. The final position depends solely on whether the total number of swaps is even or odd. As $N$ increases, accuracy scales inversely with the number of objects ($1/N$), further indicating that current VLMs do not perform genuine entity tracking but instead resort to random guessing.

\subsection{Comparison with the Perception Test}
\label{sec:compare_with_PT}
The Perception Test includes a small subset of shell-game-style clips, but several factors make it less diagnostic when evaluating visual entity tracking capability.

\textbf{Uncontrolled Complexity} In the Perception Test, difficulty varies with dataset composition in ways that directly affect performance:  (i) the number of cups varies from \textbf{2 to 4}, shifting the random baseline; and (ii) swap counts range from \textbf{0 to 16}, where 0-swap cases primarily test object permanence. As demonstrated in our previous analysis (Section~\ref{sec:swap_count} and~\ref{sec:object_count}), both object and swap counts significantly impact the performance scores.

\textbf{Visual Shortcuts} Some recordings include appearance cues (Figs.~\ref{fig:distinct_cups} to~\ref{fig:reveal_cups}) that bypass the need for tracking. These include (i) \textbf{distinct cups}, which allow for re-identification by appearance, and (ii) \textbf{transparent cups}, which directly reveal the target.

\begin{wraptable}{r}{0.4\textwidth}
  \centering
  \scriptsize 
  \begin{tabular}{lcc}
    \toprule
    \textbf{Model} & \textbf{Filtered} & \textbf{VET-Bench} \\
    \midrule
    \emph{Random} & $0.33$ & $0.33$ \\
    Gemini-3-Pro & $0.31$ & $0.34$ \\
    Gemini-3-Flash & $0.40$ & $0.30$ \\
    Qwen3-VL-8B    & $0.34$ & $0.30$ \\
    GLM-4.6V-Flash            & $0.42$ & $0.34$ \\
    Doubao-Seed-2.0       & $0.40$ & $0.30$ \\
    \bottomrule
  \end{tabular}
  \caption{Performance comparison on the filtered subset and VET-Bench.}
\label{tab:performance_results}
\end{wraptable}

\textbf{A Filtered Subset Reveals the Perception Bottleneck.} We construct a filtered subset targeting the standard shell-game setting: three identical, opaque cups (Fig.~\ref{fig:filtering_process}). From 3,525 videos in the test split, we identify 189 cups-game clips. After removing videos with visual shortcuts and ensuring at least one swap, we obtain 107 video-question pairs; restricting further to the strict 3-cup setting and excluding zero-swap cases yields 65 videos (Fig.~\ref{fig:performance_distribution}). As shown in Table~\ref{tab:performance_results}, all evaluated models collapse to near-chance performance on this filtered subset, aligning with their performance on VET-Bench. For example, Gemini-3-Pro drops from 0.80 on the full dataset to 0.31, no better than random guessing (0.33). Overall, these results confirm that while current VLMs excel in simpler instances by exploiting visual shortcuts, they struggle on the more demanding tasks where fine-grained spatiotemporal perception is required. This suggests that robust visual entity tracking remains a primary bottleneck within the Perception Test.

\subsection{Comparison with VideoReasonBench}
\label{sec:compare_with_VRB}
VideoReasonBench \citep{liu2025videoreasonbench} includes videos of cups-game-like tasks, but the swap operations are explicitly indicated by arrows overlaid on the frames (Fig.~\ref{fig:vrb}). These visual annotations effectively serve as symbolic ``swap tokens'', allowing models to reason about the state transitions based on static in-frame cues rather than from the underlying motion. In contrast, our benchmark represents a more realistic shell-game setting, where only moving objects are visible with no frame-level cues. Correctly solving the task therefore requires exploiting spatiotemporal continuity across frames. This design aligns with many recent video benchmarks (Section~\ref{sec:related_work}) where the decisive information is encoded in the temporal dynamics between frames, not within a single frame. Consequently, models like Gemini-2.5-Pro can achieve $56\%$ on VideoReasonBench while remaining near-chance on VET-Bench, where no explicit swap cues are provided.

\section{Theoretical Analysis}
Humans solve the shell game through an intuitive perceptual process requiring little conscious effort, yet current VLMs consistently fail at this task. This raises a fundamental question: is a transformer-based VLM inherently capable of solving the shell game? We address this question by analyzing the computational complexity of the decision version of visual entity tracking, defined as follows.

\begin{definition}[Visual Entity Tracking, $\mathrm{TRACK}_k$]
$\mathrm{TRACK}_k$ is the problem of tracking $k$ visually indistinguishable objects in a video $V = (F_0, \dots, F_T)$ of $T+1$ frames on an $H \times W$ grid, where $k$, $H$, and $W$ are constants. The input is assumed to satisfy localization and continuity conditions given in Appendix~\ref{app:proof}. Let $\pi$ be the global permutation that maps the $k$ objects from their initial lexicographic ordering of positions in frame $0$ to their final ordering in frame $T$. The problem asks whether $\pi$ is the identity permutation.
\end{definition}

% \begin{definition}[Visual Entity Tracking, $\mathrm{TRACK}_k$]
% $\mathrm{TRACK}_k$ is the problem of tracking $k$ visually indistinguishable objects in a video $V = (F_0, \dots, F_T)$ of $T+1$ frames on an $H \times W$ grid, where $k,H,W$ are constants. Assuming localization and continuity conditions (given in Appendix~\ref{app:proof}), the unique matching between objects in consecutive frames induces a permutation $\pi_t \in S_k$ mapping each object’s index under the lexicographic ordering of object positions in frame $t$ to its index in frame $t+1$. Let $\pi = \pi_{T-1} \circ \dots \circ \pi_0$ be the global permutation that maps the $k$ objects from their initial lexicographic ordering in frame $0$ to their final ordering in frame $T$. The decision version of $\mathrm{TRACK}_k$ asks whether $\pi$ is the identity permutation.
% \end{definition}

Drawing inspiration from state tracking problems \citep{liu2023transformers, merrill2024illusion}, we characterize the complexity of $\mathrm{TRACK}_k$ via the word problem for the symmetric group $S_5$, a canonical $\mathbf{NC}^1$-complete problem \citep{barrington1989bounded}.

\begin{definition}[Word Problem for $S_5$, $\mathrm{WORD}_{S_5}$]
Let the generators be the adjacent transpositions $\tau_j = (j, j+1)$ for $j \in \{1, 2, 3, 4\}$. Given a sequence $W = \langle \sigma_1, \dots, \sigma_N \rangle$ where each $\sigma_i \in \{\tau_1, \tau_2, \tau_3, \tau_4\}$, let $\Pi = \sigma_N \circ \dots \circ \sigma_1$. The output is \textsc{True} iff $\Pi$ is the identity.
\end{definition}

% \begin{theorem}
% For any fixed $k \ge 5$, $\text{TRACK}_k$ is $\mathbf{NC}^1$-complete.
% \end{theorem}
\begin{restatable}{theorem}{completeness} \label{thm:completeness}
For any fixed $k \ge 5$, $\mathrm{TRACK}_k$ is $\mathbf{NC}^1$-complete.
\end{restatable}

\begin{proof}[Proof Sketch]
Membership in $\mathbf{NC}^1$ follows by computing, for each adjacent pair of frames, the unique permutation $\pi_t\in S_k$ induced by the localization and continuity conditions, and then composing
$\pi=\pi_{T-1}\circ\cdots\circ\pi_0$
in $\NC{1}$ (Lemma~\ref{lem:1}). Hardness is established via reduction from $\mathrm{WORD}_{S_5}$ by constructing a video that physically realizes adjacent transposition generators (Lemma~\ref{lem:2}). Full proofs are provided in Appendix~\ref{app:proof}.
\end{proof}

Prior work shows that constant-depth transformers can be simulated within the circuit class $\mathbf{TC}^0$ \citep{merrill2023parallelism, merrill2024the}. Assuming the widely held conjecture that $\mathbf{TC}^0 \subsetneq \mathbf{NC}^1$, fixed-depth transformers are theoretically limited to solving $\mathbf{NC}^1$-complete problems like $\mathrm{TRACK}_k$ for $k \ge 5$ on arbitrary-length sequences. The $\mathbf{NC}^1$-completeness result therefore provides a theoretical foundation for the need for CoT in visual entity tracking \citep{feng2023towards, li2024chain, huang2025transformers}.

By the Krohn--Rhodes theorem \citep{krohn1965algebraic}, the complexity of a word problem depends on the algebraic structure of the underlying group. For $k \ge 5$, $S_k$ is non-solvable, and its word problem is $\mathbf{NC}^1$-complete \citep{barrington1989bounded}. In contrast, $S_2$ is a cyclic group, and the word problem reduces to a parity task, which lies in $\mathbf{TC}^0$. Our hardness result holds specifically for $k \ge 5$; for smaller object counts, the task may admit shortcut solutions \citep{liu2023transformers}. Nevertheless, even for $k=2$ (e.g. parity and coin-flip tasks), prior work \citep{wei2022chain, anil2022exploring, wies2023subtask, kim2025transformers} consistently shows that effective length generalization still requires intermediate supervision or CoT.

If objects possess unique visual identifiers, the task collapses from a sequential state-tracking problem into a parallelizable visual search problem ($\mathbf{AC}^0$). This explains why tracking objects with distinct appearances is often easier: models can rely on appearance-based rematching rather than entity tracking. This is analogous to the difference in hardness between source-target notation (UCI) and standard algebraic notation (SAN) \citep{toshniwal2022chess} for chess state tracking, as discussed by \cite{merrill2024illusion}.

While humans often perceive visual entity tracking as an intuitive perceptual ability, our analysis suggests it belongs to the same complexity class as other state-tracking tasks, such as tracking chess moves or tracking entities in a narrative \citep{merrill2024illusion}, typically considered difficult due to their latent reasoning requirements. We posit that the human visual system overcomes this limitation via foveal eye movements, which can be viewed as a form of ``physical CoT'' by explicitly updating the state across frames. 

\paragraph{Training with Direct Answer Fails in Practice}
\label{sec:training}

\begin{wrapfigure}{r}{0.4\textwidth} % {r} for right side, 0.4 is the width
  \centering
    \includegraphics[width=0.38\columnwidth]{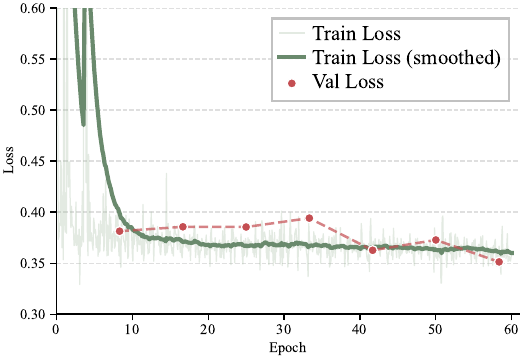}
    \caption{Training and validation loss for direct-answer training on 500 synthetic VET-Bench cups-game videos.}
  \label{fig:training}
\end{wrapfigure}

To empirically verify that VLMs struggle to learn visual entity tracking from direct-answer supervision without CoTs, we train Qwen2.5-VL-3B-Instruct on 500 synthesized cups-game videos (varying only in shuffle permutations) with an 8 FPS sampling rate. As shown in Fig.~\ref{fig:training}, the loss remains stagnant at the level of random chance even after 60 epochs. This phenomenon mirrors the difficulty of learning the parity task \citep{hahn2024sensitive, kim2025transformers}. Although parity is in $\mathbf{TC}^0$ and therefore expressible by transformers, the target label lacks low-order statistical correlations with input features, making it difficult to train end-to-end using gradient-based algorithms. In our case, the long sequence of video tokens makes it difficult for the model to map the input directly to the final label. Consequently, the model fails to capture the underlying dynamics and instead resorts to predicting based on label frequency within the training distribution.

\section{Spatiotemporal Grounded Chain-of-Thought}
\label{sec:sgcot}
Previous analysis suggests CoT is essential for solving general visual entity tracking tasks. While models like Gemini-3 produce logically valid CoTs, they still perform near-random on VET-Bench due to inaccurate perceptual grounding (Section~\ref{sec:hallucination}). To address this, we propose Spatiotemporal Grounded Chain-of-Thought (SGCoT). We leverage Molmo2 \citep{clark2026molmo2} for this task, as it is extensively trained on video object tracking and exhibits strong grounding capabilities that can be repurposed for visual entity tracking QA.

\paragraph{Limitations of the Base Model}
 When prompted for tracking, Molmo2 can generate the referred object's trajectory in the following structured format: 
 % \texttt{<tracks coords="timestamp object\_idx x y;...">Object</tracks>},
\begin{quote}
    \texttt{<tracks coords="timestamp object\_idx x y;...">Object</tracks>}
\end{quote}
where timestamps are spaced at 0.5-second intervals, and $x$, $y$ are normalized coordinates 0--1000. However, this behavior resembles a specialized tracking tool rather than a general-purpose instruction-following VLM. When prompted with \emph{``Where is the Queen of Hearts at the end of the video?''}, Molmo2 responds in direct-answer mode (e.g., \emph{``Left.''}) without explicitly performing spatiotemporal tracking, resulting in near-random performance (Fig.~\ref{fig:molmo2_direct_answer}). Our goal is therefore to leverage Molmo2’s native tracking capability, using generated trajectories as CoTs for question answering.

% \begin{center}
% \texttt{<tracks coords="timestamp objectID x y;...">Object</tracks>}
% \end{center}
\paragraph{Training for SGCoT Alignment}
To avoid the memory and compute cost of high-FPS video input, we propose a simple and highly efficient approach that aligns Molmo2-SGCoT to solve VET-Bench tasks end-to-end using minimal text-only training samples without video input. We adopt CoT prompting (e.g., by prefixing the query with ``Track [object]'') to elicit tracking (Fig.~\ref{fig:example_of_Molmo2-SGCoT}). Here is an example of Molmo2-SGCoT's prompt and response:

\begin{tcolorbox}[
    colback=gray!5, 
    colframe=gray!50, 
    arc=1mm, 
    boxrule=0.5pt, 
    left=4pt, right=4pt, top=4pt, bottom=4pt,
    fonttitle=\bfseries,
    toptitle=1pt, 
    bottomtitle=1pt,
    % title=Example of Molmo2-SGCoT on VET-Bench Card Game
]
    \small
    Track the Queen of Hearts and answer where is the Queen of Hearts at the end of the video?
    \vspace{0.5em}
    \hrule
\begin{Verbatim}[
    fontsize=\small, 
    breaklines=true, 
    breakanywhere=true, 
    breaksymbol={},
    commandchars=\\\{\} % This allows LaTeX commands inside Verbatim
]
<tracks coords="0.0 1 745 512; 0.5 1 745 512; 1.0 1 745 500; 1.5 1 742 425; 2.0 1 681 296; ... 12.0 1 \underline{247 500}">the Queen of Hearts</tracks> \textcolor{red}{Answer: left.}
\end{Verbatim}
\end{tcolorbox}
Each sample contains a synthesized \verb|<tracks></tracks>| trajectory as SGCoT, generated by a Python script. In VET-Bench, the final answer (``left'', ``middle'', or ``right'') depends only on the terminal location of the tracked object, which is explicitly defined by the final coordinate (underlined) in the SGCoT. During training, we mask the loss on all synthesized trajectory tokens within \verb|<tracks></tracks>| and supervise only the final answer (shown in \textcolor{red}{red}), encouraging the model to retain its grounding capability while learning to generate the final answer from the SGCoT. Notably, our method requires no explicit training on VET-Bench videos. While any randomly generated trajectory works, we find that using in-distribution synthesized trajectories causes minimal degradation of tracking ability. Our SGCoT alignment is highly efficient; one epoch over 300 samples completes in 3 minutes on a single A100 GPU, improving Molmo2's accuracy from near-random to 91\%. Training details are provided in Appendix~\ref{app:training}.

\paragraph{Error Analysis}
Despite its high accuracy, Molmo2-SGCoT occasionally fails. As shown in Fig.~\ref{fig:sgcot_wrong}, errors typically arise during the SGCoT perception stage when the model misidentifies visually identical objects. This manifests as abrupt ``jumps'' in the output trajectory, resulting in incorrect terminal locations and subsequent failures in the final answer.

\paragraph{Why SGCoT Succeeds and Other CoTs Fail}
\label{sec:why_sgcot_succeed}
The success of SGCoT is likely due to its fine-grained intermediate state representation. In particular, it explicitly aligns (i) when an event occurs (using discrete timestamps based on frame sampling) and (ii) where each entity is located (using fine-grained spatial coordinates). This combination yields an unambiguous state representation, better suited to long-horizon tracking. In contrast, many generic descriptive CoTs, as shown in Figs.~\ref{fig:gemini-3-pro-cup} to~\ref{fig:model_performances}, are loosely grounded in time and space. CoTs that ground events with coarse time ranges (e.g., ``00:01–00:02'') are inherently brittle; fine-grained motion events often have non-integer start and end times. These coarse intervals lead to temporal misalignment \citep{du2025svlta}, causing errors to accumulate. Similarly, CoTs that rely on coarse referring expressions (e.g., ``left'', ``right'') underspecify the state. SGCoT avoids both failure modes by discretizing time at consistent intervals and ensuring that a precise spatial state corresponds to each timestamp.

\section{Related Work}
\paragraph{Video Benchmarks for Spatiotemporal Perception}
\label{sec:related_work}
Existing video QA benchmarks \citep{zhou2025mlvu, song2025video, zhao2025mmvu, fang2024mmbench, li2024mvbench, li2023seed} mainly focus on general video understanding and reasoning. However, many fail to decouple temporal dynamics from static cues, enabling models to achieve high performance via shortcuts \citep{lei2023revealing}. 
Recent benchmarks---such as TempCompass \citep{liu2024tempcompass}, TemporalBench \citep{cai2024temporalbench}, E.T. Bench \citep{liu2024bench}, TVBench \citep{cores2024lost}, VideoVista \citep{li2024videovista}, TOMATO \citep{shangguan2024tomato}, MotionBench \citep{hong2025motionbench}, Tuna \citep{kong2025tuna}, Know-Show \citep{sugandhika2025know}, EgoTempo \citep{plizzari2025omnia}, VCR-Bench \citep{qi2025vcr}, MVP \citep{krojer2025shortcut}, AoTBench \citep{xue2025seeing}, SpookyBench \citep{upadhyay2025time}, Timeblind \citep{li2026timeblind}, NarrativeTrack \citep{ha2026narrativetrack}---expose the limitation that VLMs often rely on single-frame biases or language priors without genuinely modeling temporal dynamics. Following this line of work, our VET-Bench focuses on the fine-grained spatiotemporal perception of VLMs.

\paragraph{Reasoning with Grounding}
Recent VLMs \citep{bai2025qwen3, clark2026molmo2, deitke2025molmo} demonstrate strong grounding abilities through targeted training \citep{lin2014microsoft, yu2016modeling}. This facilitates reasoning with grounding to improve robustness and reliability \citep{jiang2025rex}. VoCoT \citep{li2025vocot} uses spatially anchored objects to guide multi-step reasoning. GCoT \citep{wu2025grounded, xia2025bootstrapping} explicitly injects bounding boxes into CoT to enhance faithfulness and reduce hallucination. We extend this to the temporal domain and show that SGCoT can reliably solve visual entity tracking tasks.

\paragraph{Entity Tracking} Entity tracking has been studied in textual domains \citep{toshniwal2022chess, kim2023entity}, with a focus on evaluating LLMs’ abilities to maintain representations of entity states. \cite{merrill2024illusion} prove that chess state tracking in UCI notation is $\mathbf{NC}^1$-complete. Building on their constructive proof, we show that visual entity tracking also falls into this class and necessitates CoT for generalization. MET-Bench \citep{cohen2025met} studies entity tracking in shell game and chess using image and text inputs. In contrast, our work addresses the visual entity tracking problem in videos.

\section{Limitations and Future Work}
We study a simplified setting in VET-Bench where the final answer is fully determined by the location information encoded in the generated SGCoT. In real-world scenarios, however, the final answer may also depend on visual evidence in the scene---for example, resolving general referring expressions (\emph{``Which cup contains the ball at the end of the video, from the person's point of view?''}), which require additional visual perception within the final frames. Future work could explore integrating SGCoT with complex referring expressions across arbitrary frames. Moreover, our analysis assumes localization and continuity conditions to ensure swap movements remain unambiguous. In practice, tracking is often more challenging: objects move in close proximity, overlap, exhibit motion blur, or undergo occlusion \citep{qi2022occluded, pothiraj2025capture}. Handling these complex scenarios may require incorporating models with stronger physical priors or world models to enable more robust tracking.

\section{Conclusion}
\label{sec:conclusion}
In this work, we identified visual entity tracking as a fundamental bottleneck in current VLMs. We introduced VET-Bench and demonstrated that state-of-the-art models perform near random chance when appearance cues are removed. Our theoretical analysis established that visual entity tracking is $\mathbf{NC}^1$-complete, suggesting that fixed-depth VLMs are inherently limited in expressivity when solving such tasks without CoT. We empirically verified that VLMs are unable to solve the shell game in direct-answer mode, even with sufficient training. Finally, by aligning Molmo2 to generate SGCoT before the final answer, we achieved state-of-the-art accuracy exceeding 90\% on VET-Bench. Ultimately, we hope this work paves the way for future VLMs to leverage SGCoT in broader video perception and reasoning.

\bibliography{colm2026_conference}
\bibliographystyle{colm2026_conference}

\clearpage
\appendix

\section{Detailed Audit of the Perception Test for Visual Entity Tracking}

\begin{figure}[H]
 \centering
 \begin{subfigure}[b]{0.48\textwidth}
     \centering
     \includegraphics[width=\textwidth]{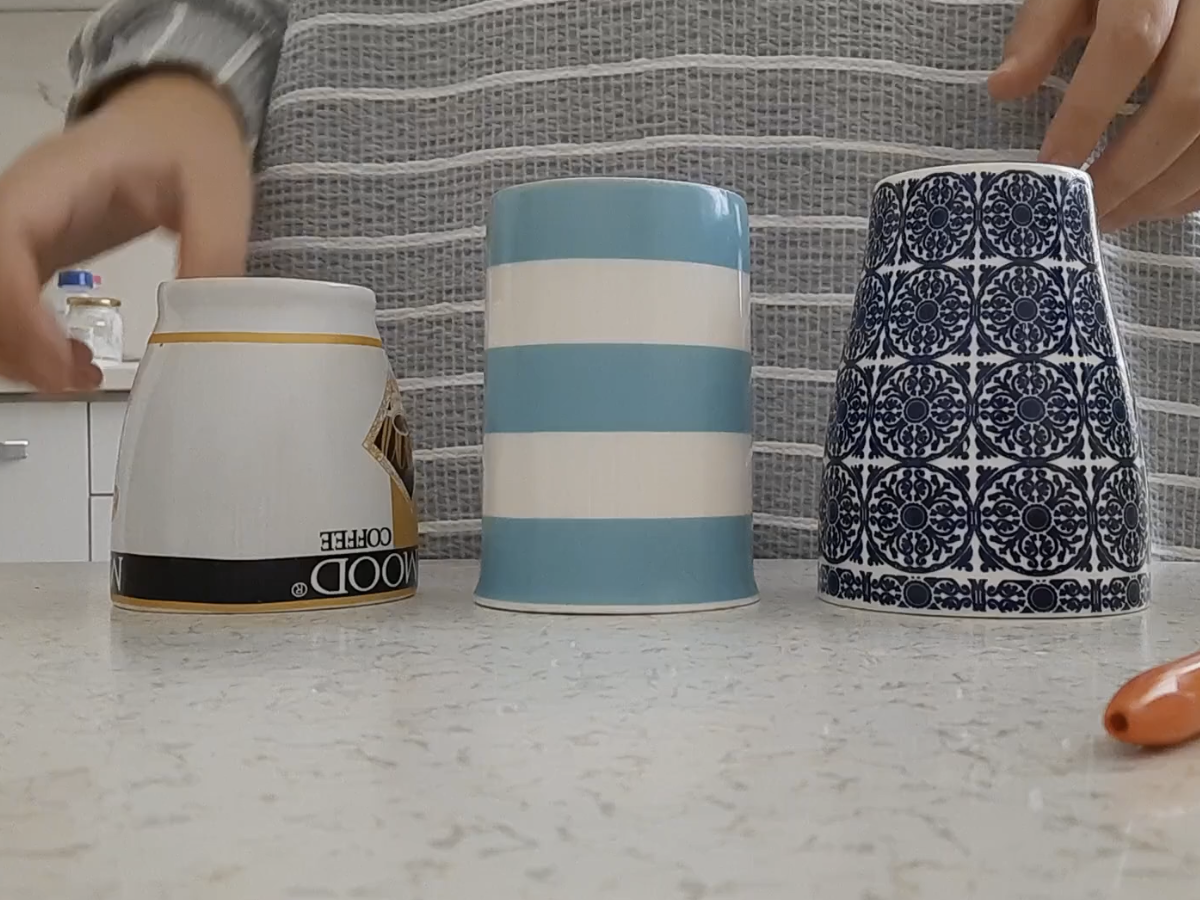}
     % \caption{Caption for first image}
 \end{subfigure}
 \hfill
 \begin{subfigure}[b]{0.48\textwidth}
     \centering
     \includegraphics[width=\textwidth]{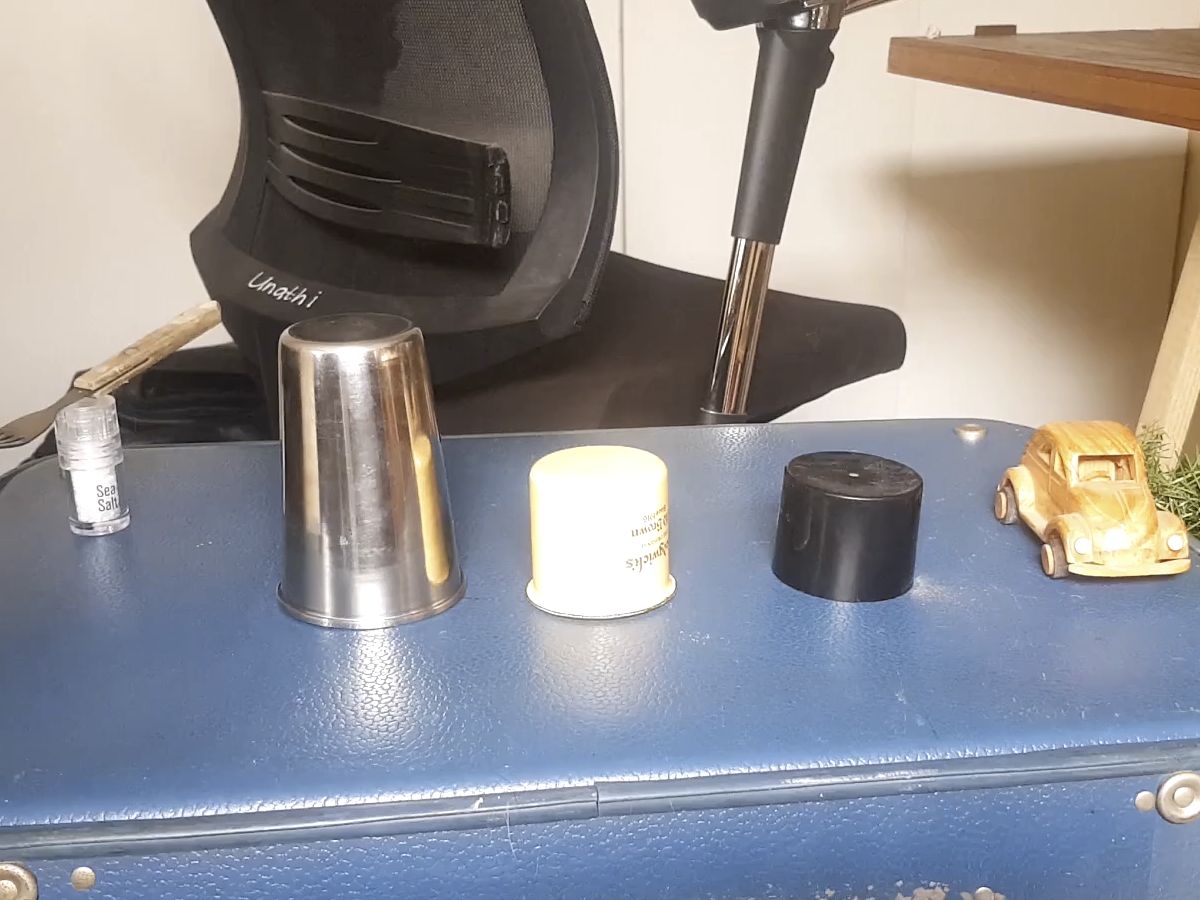}
     % \caption{Caption for second image}
 \end{subfigure}
 \caption{Example frames from videos involving distinct cups in the Perception Test.}
 \label{fig:distinct_cups}
\end{figure}

\begin{figure}[H]
 \centering
 \begin{subfigure}[b]{0.48\textwidth}
     \centering
     \includegraphics[width=\textwidth]{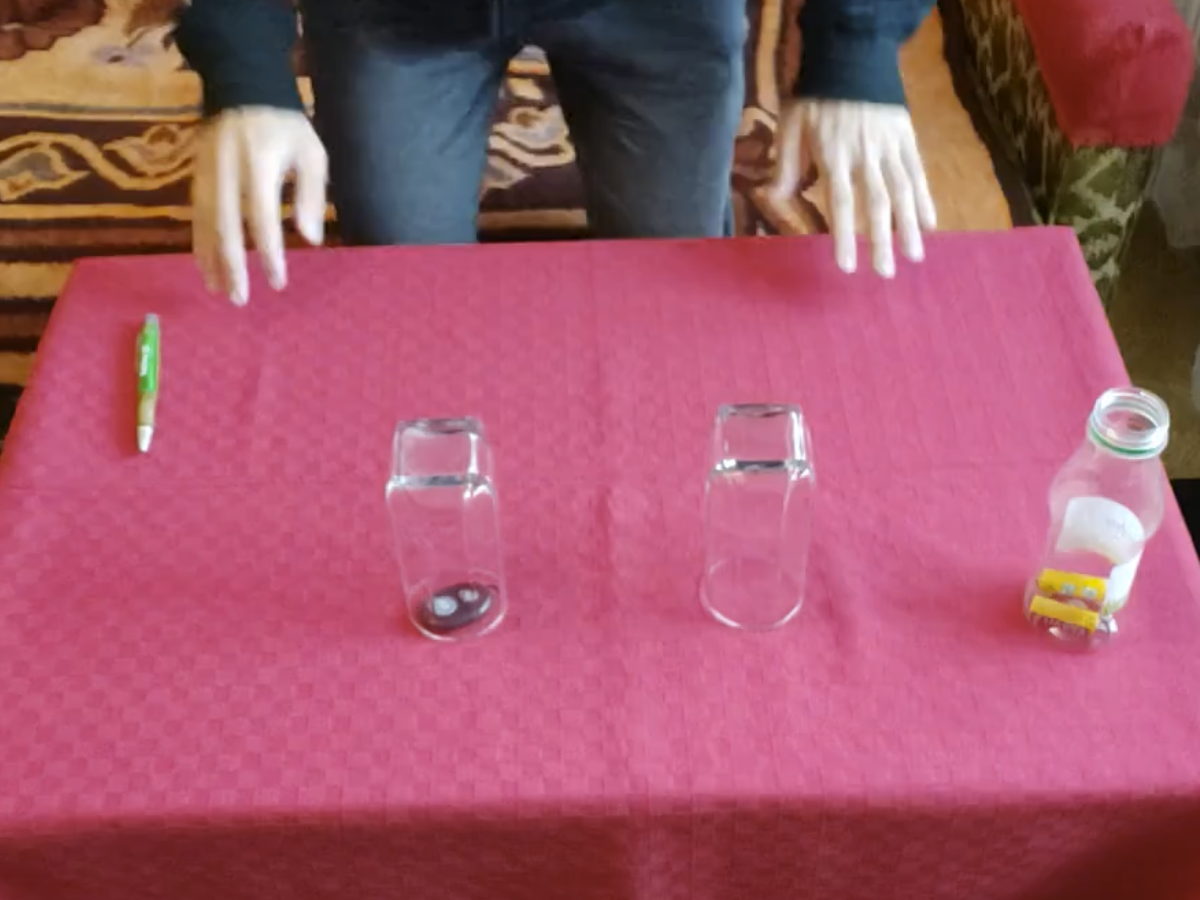}
     % \caption{Caption for first image}
 \end{subfigure}
 \hfill
 \begin{subfigure}[b]{0.48\textwidth}
     \centering
     \includegraphics[width=\textwidth]{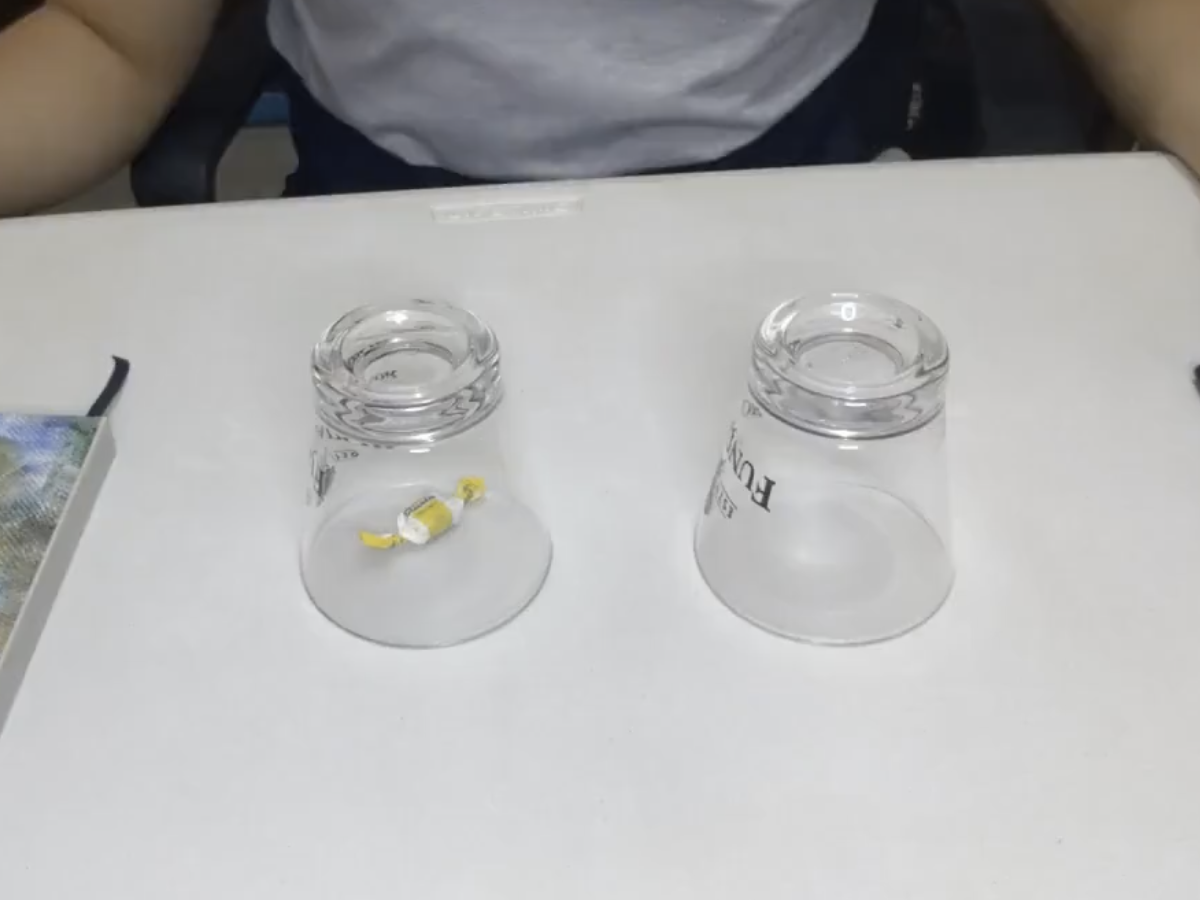}
     % \caption{Caption for second image}
 \end{subfigure}
 \caption{Example frames from videos involving transparent cups in the Perception Test.}
 \label{fig:transparent_cups}
\end{figure}

\begin{figure}[H]
 \centering
 \begin{subfigure}[b]{0.48\textwidth}
     \centering
     \includegraphics[width=\textwidth]{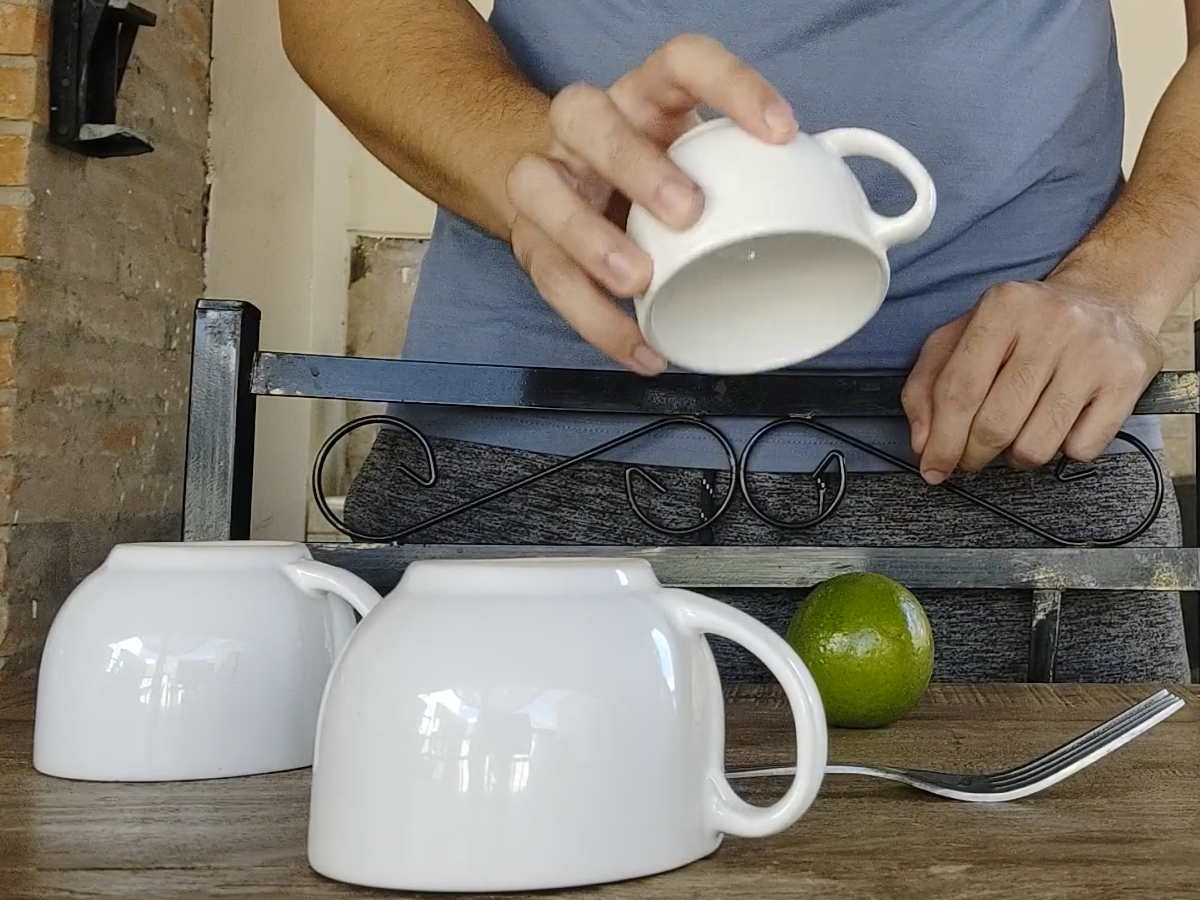}
     % \caption{Caption for first image}
 \end{subfigure}
 \hfill
 \begin{subfigure}[b]{0.48\textwidth}
     \centering
     \includegraphics[width=\textwidth]{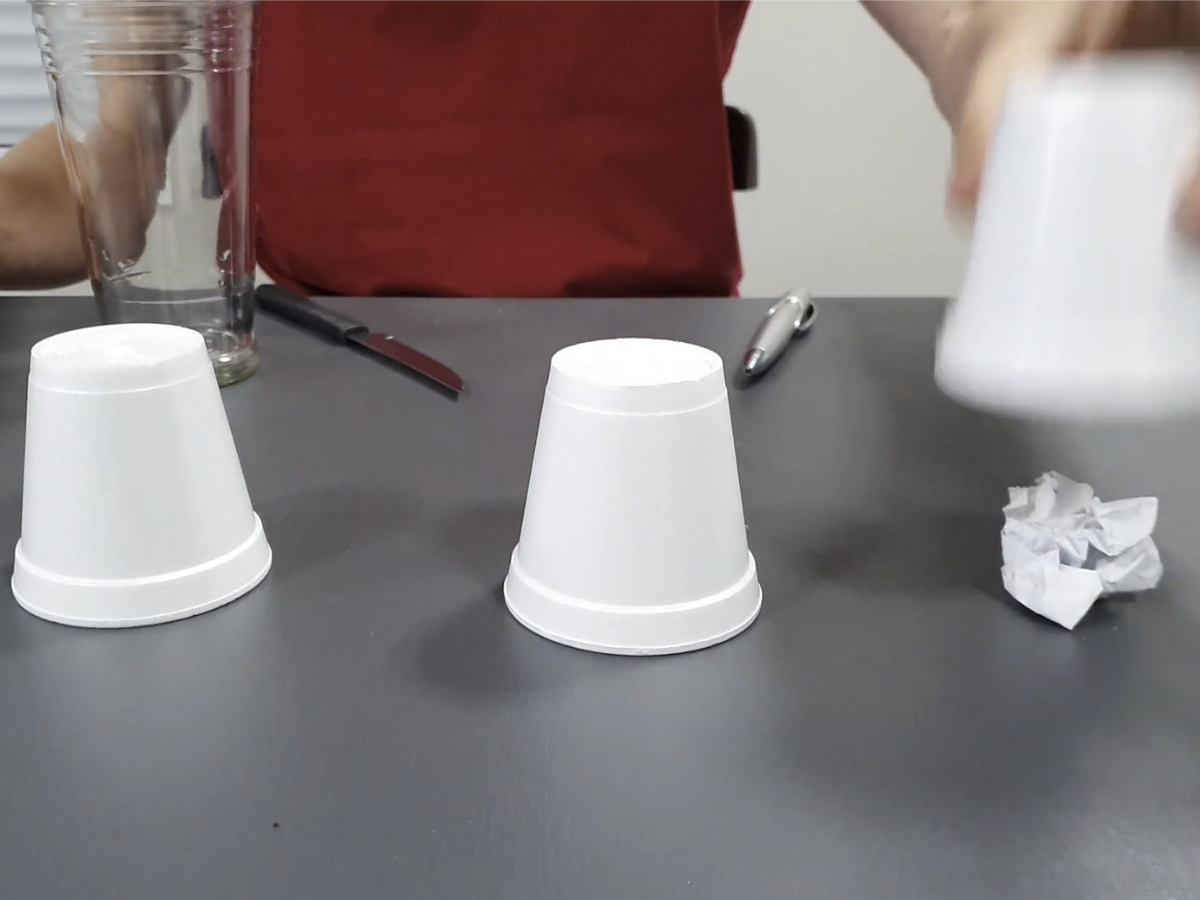}
     % \caption{Caption for second image}
 \end{subfigure}
 \caption{Visual shortcuts in videos \textbf{\#7632} and \textbf{\#9896}. These videos contain unedited cut frames at the end that reveal the answer by showing the cups being lifted. This unintended disclosure provides a shortcut that allows the model to simply look at the final state to identify the object's location, making the tracking of the shuffle irrelevant.}
 \label{fig:reveal_cups}
\end{figure}

\begin{figure}[H]
 \centering
 \begin{subfigure}[b]{0.48\textwidth}
     \centering
     \includegraphics[width=\textwidth]{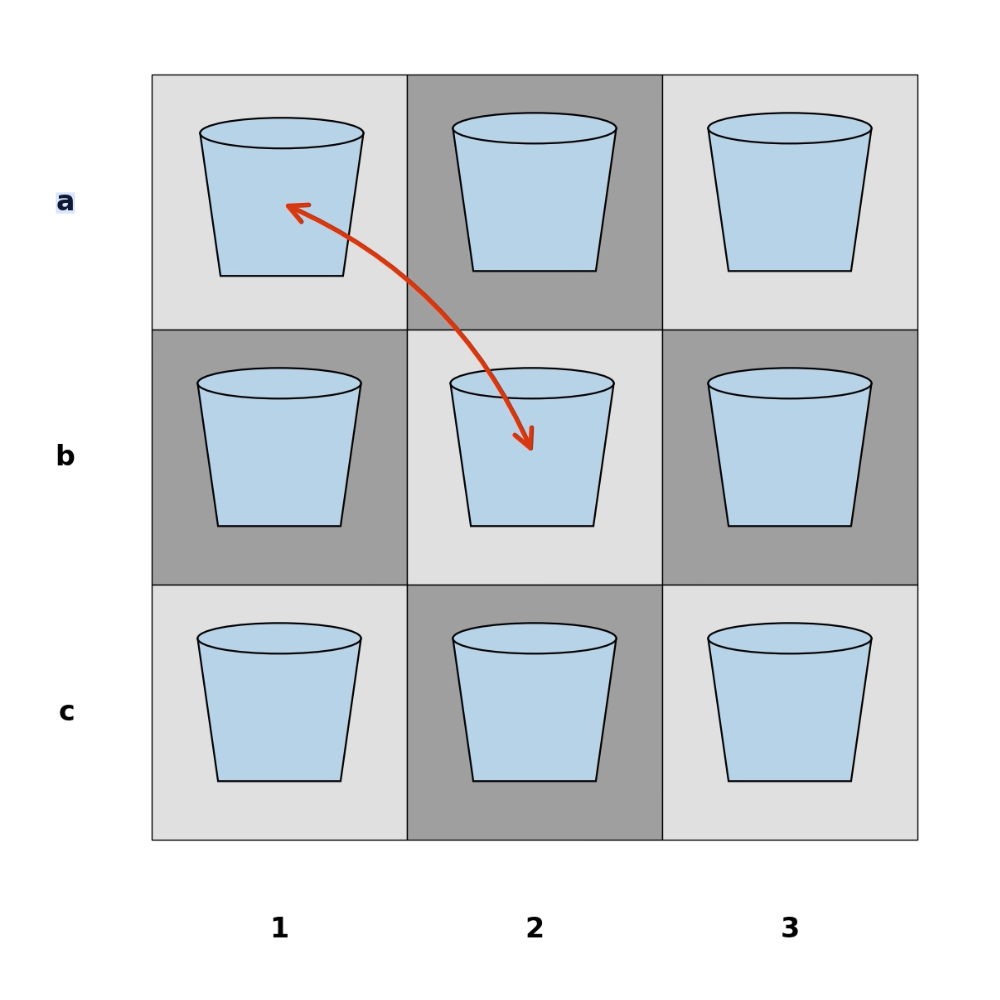}
     % \caption{Caption for first image}
     \label{fig:image1}
 \end{subfigure}
 \hfill
 \begin{subfigure}[b]{0.48\textwidth}
     \centering
     \includegraphics[width=\textwidth]{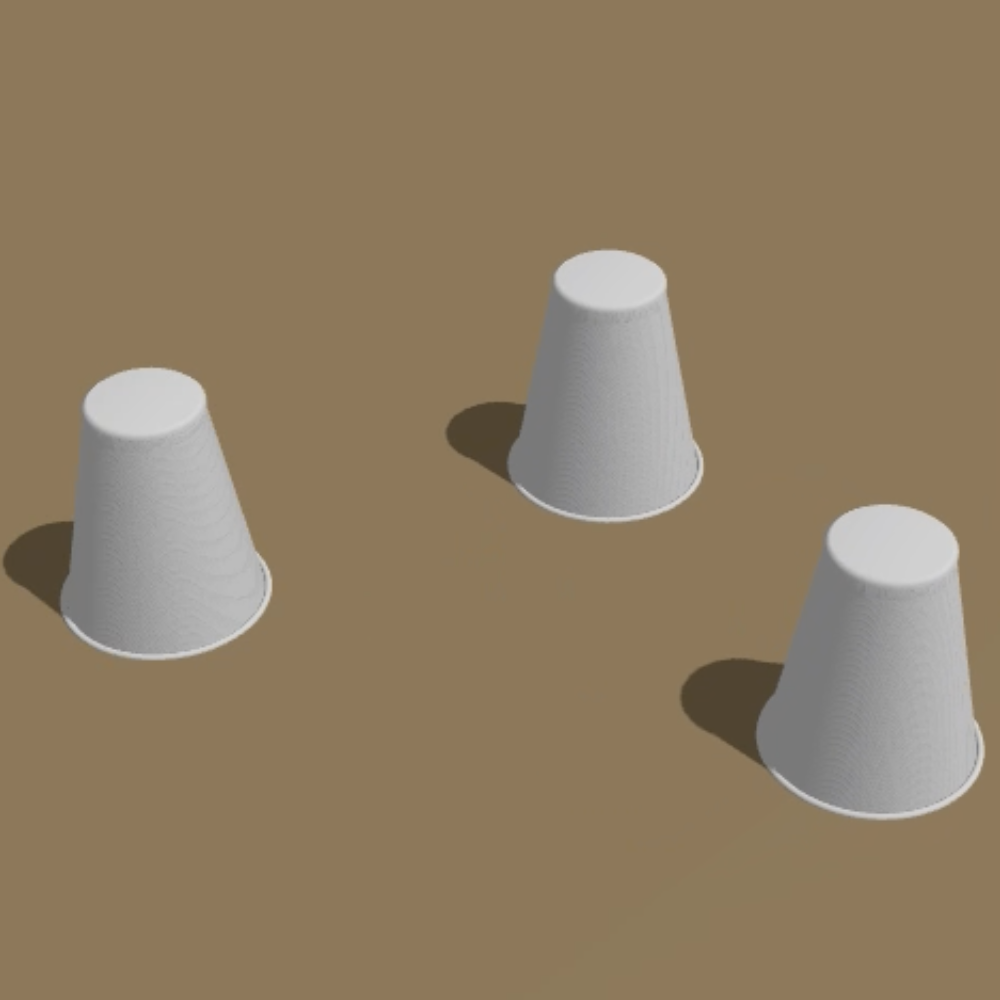}
     % \caption{Caption for second image}
     \label{fig:image2}
 \end{subfigure}
 \caption{Comparison of the cups-game tasks in VideoReasonBench (left) and VET-Bench (right). VideoReasonBench provides explicit frame-level cues highlighting swap operations (indicated by red arrows), which are absent in VET-Bench.}
 \label{fig:vrb}
\end{figure}

\begin{figure}[H]
    \centering
    \includegraphics[width=0.6\columnwidth]{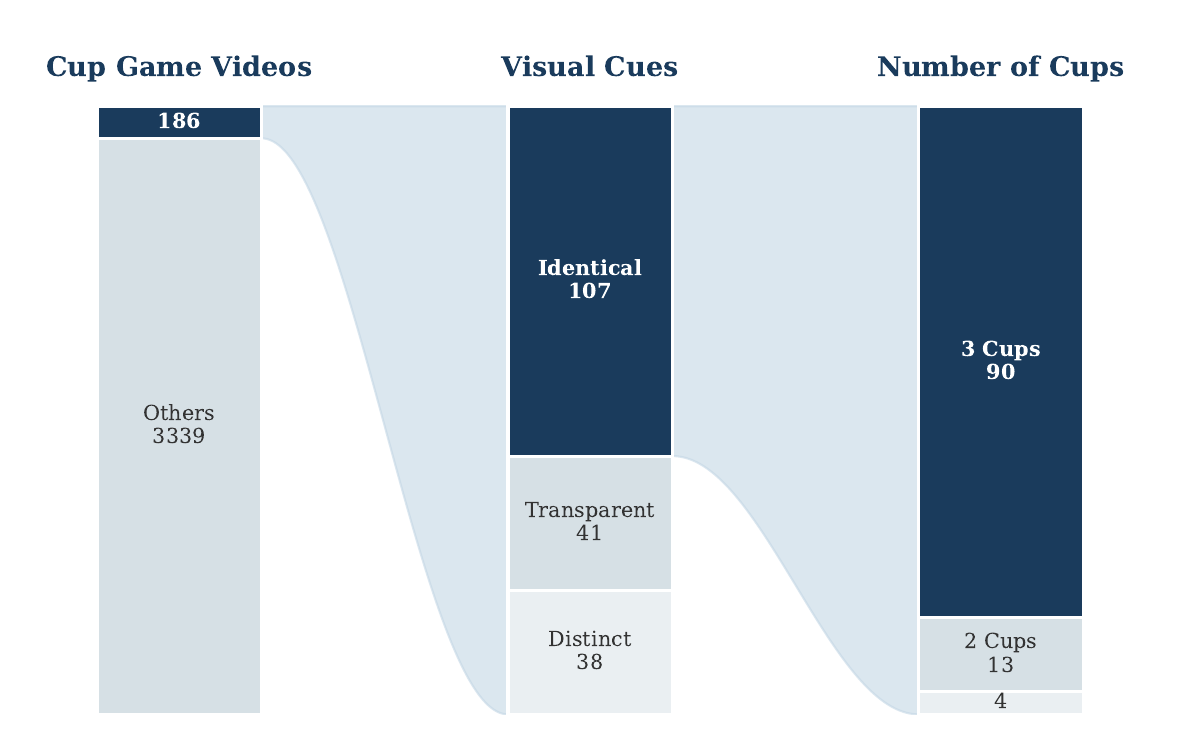}
    \caption{Filtering process for the identical 3-cup setting in the Perception Test test split.}
    \label{fig:filtering_process}
\end{figure}

\begin{figure}[H]
    \centering
    \includegraphics[width=0.8\columnwidth]{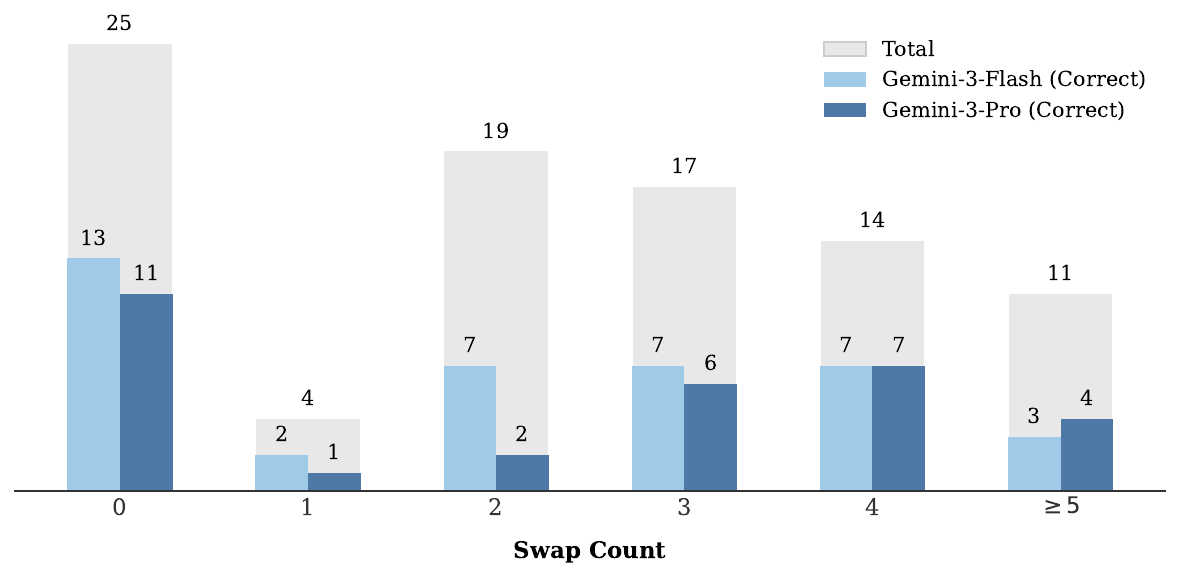}
    \caption{Distribution of swap counts in the filtered Perception Test subset, along with the number of correct responses from Gemini-3-Pro and Gemini-3-Flash.}
    \label{fig:performance_distribution}
\end{figure}

\clearpage
\section{Proof of Theorem 1}
\label{app:proof}

Recall that the decision problem $\mathrm{TRACK}_k$ takes as input a video $V=(F_0,\dots,F_T)$ of $T{+}1$ frames on an $H\times W$ pixel grid, where $k$, $H$, and $W$ are constants and the input length is
$n=\Theta(T)$, and asks whether the permutation $\pi$ from frame $0$ to frame $T$ is the identity.

The input is assumed to satisfy the following conditions:

\begin{enumerate}
    \item \textbf{Localization.}
    Each frame $F_t$ contains exactly $k$ object centers (grid coordinates), denoted as $C_t = \{c_{t,1}, \dots, c_{t,k}\}$ for $t \in \{0, \dots, T\}$, where $c_{t,i} \in [H] \times [W]$, and the $k$ object centers can be recovered from each frame in $\mathbf{NC}^1$.

    \item \textbf{Continuity.}
    Each object center moves by at most a fixed distance $d$ between consecutive frames, and any two centers within any single frame are at distance $>2d$, where, for convenience, we take $\mathrm{dist}(\cdot,\cdot)$ to be the $L_1$ (Manhattan) distance.
\end{enumerate}

The conditions above are standard prerequisites to ensure that the problem is well-posed in controlled environments and are easily satisfied in practice. The localization condition ensures that object centers are efficiently recoverable; for example, using constant-size visual markers allows template matching in $\mathbf{NC}^1$. The continuity condition formalizes the constraint discussed earlier (Section~\ref{sec:tasks} and~\ref{sec:settings}) to prevent ambiguity by guaranteeing a unique correspondence between objects in consecutive frames. Practically, this can be achieved by using a sufficiently high frame rate relative to object speeds, so that objects move only a small distance between adjacent frames. \\

% Let $\hat c_{t,1},\dots,\hat c_{t,k}$ be the centers in frame $t$ ordered lexicographically by $(x,y)$.
% Define $\pi_t\in S_k$ by letting $\pi_t(i)=j$ iff $\hat c_{t,i}$ matches the unique center $\hat c_{t+1,j}$ in frame $t{+}1$ within distance $d$. Let $\pi=\pi_{T-1}\circ\cdots\circ\pi_0$.
% $\mathrm{TRACK}_k$ outputs \textsc{True} if and only if $\pi$ is the identity permutation.

\begin{lemma}
\label{lem:1}
$\mathrm{TRACK}_k \in \mathbf{NC}^1$.
\end{lemma}

\begin{proof}

We show that the permutation $\pi$ can be computed in $\mathbf{NC}^1$.

For each frame $t$, let the localization procedure output the set of $k$ distinct centers as grid coordinates.
Since $H$ and $W$ are constants, each coordinate has constant bit-length. Hence, per-frame computations are constant-size.

Let $\hat c_{t,1},\dots,\hat c_{t,k}$ be the centers sorted lexicographically by $(x,y)$ (i.e., by $x$ coordinate, breaking ties by $y$).
Since both $k$ and the coordinate length are constant, this sorting is computable by a constant-size circuit for each frame.

By the continuity condition, each $\hat c_{t,i}$ has at least one successor $\hat c_{t+1,j}$ within distance $d$.
Such a $j$ is unique. Suppose there exist distinct indices $j\neq j'$ such that
\[
\mathrm{dist}(\hat c_{t,i},\hat c_{t+1,j})\le d
\qquad\text{and}\qquad
\mathrm{dist}(\hat c_{t,i},\hat c_{t+1,j'})\le d,
\]
then
\[
\mathrm{dist}(\hat c_{t+1,j},\hat c_{t+1,j'})\le 2d,
\]
contradicting the continuity condition in frame $t+1$.
Define $\pi_t:[k]\to[k]$ by $\pi_t(i)=j$.

$\pi_t$ is injective: Suppose $\pi_t(i)=\pi_t(i')=j$ for some $i\neq i'$.
Then
\[
\mathrm{dist}(\hat c_{t,i},\hat c_{t+1,j})\le d
\qquad\text{and}\qquad
\mathrm{dist}(\hat c_{t,i'},\hat c_{t+1,j})\le d,
\]
so
\[
\mathrm{dist}(\hat c_{t,i},\hat c_{t,i'})\le 2d,
\]
again contradicting the continuity condition, now in frame $t$.
Thus $\pi_t$ is injective. Since $\pi_t$ is an injective map from a finite set of size $k$ to itself, it is bijective, and hence $\pi_t\in S_k$.

Because the coordinates have constant bit-length and $k$ is fixed, each predicate
$\mathrm{dist}(\hat c_{t,i},\hat c_{t+1,j})\le d$
is computable by a constant-size circuit, and so is the selection of the unique $j$ for each $i$.
Hence, each $\pi_t$ is computable by a constant-size circuit.

The induced global permutation is
\[
\pi=\pi_{T-1}\circ\cdots\circ\pi_0.
\]
Since $k$ is fixed, composing two permutations in $S_k$ is computable by a constant-size circuit.
Computing this product in a balanced binary tree gives depth $O(\log T)=O(\log n)$ and polynomial size.

Finally, the localization step is in $\mathbf{NC}^1$ by assumption, and all per-frame sorting and matching operations are constant-size, repeated over $T=\Theta(n)$ frames. Thus the total size is polynomial, and the overall depth is $O(\log n)$. Therefore, the global permutation can be computed in $\mathbf{NC}^1$.

Since $k$ is fixed, checking whether the resulting permutation $\pi$ is the identity is computable by a constant-size circuit. Hence, $\mathrm{TRACK}_k \in \mathbf{NC}^1$.
\end{proof}

\begin{lemma}
\label{lem:2}
$\mathrm{WORD}_{S_5} \le_{\mathbf{AC}^0} \mathrm{TRACK}_k$ for any fixed $k \ge 5$.
\end{lemma}

\begin{proof}
We reduce from the word problem over $S_5$ with generators
\[
\tau_1=(12),\qquad \tau_2=(23),\qquad \tau_3=(34),\qquad \tau_4=(45),
\]
which is $\mathbf{NC}^1$-complete.

Fix constants $d>0$, $L>2d$, and $h>d$. Place the first five anchor positions at
\[
p_i=((i-1)L,0),\qquad i=1,2,3,4,5,
\]
so they lie on the same horizontal row and consecutive anchors are distance $L$ apart. Place the remaining anchors
\[
p_6,\dots,p_k
\]

at fixed positions at distance greater than $2d$ from the region traversed by the gadget trajectories. Thus all inactive objects can remain stationary throughout the video without violating the continuity condition.

We use only the first five anchors to simulate the generators of $S_5$; the remaining anchors are occupied by stationary objects. The objects are rendered as fixed constant-size bright templates on a black background, so the object centers can be recovered trivially and the localization condition is satisfied.

For each $i\in\{1,2,3,4\}$, we construct a constant-length gadget $G_i$ that swaps the objects at $p_i$ and $p_{i+1}$ and leaves all other objects fixed:
\[
p_i=((i-1)L,0),\qquad p_{i+1}=(iL,0).
\]
The gadget moves the object at $p_i$ along the three-segment path
\[
((i-1)L,0)\to ((i-1)L,h)\to (iL,h)\to (iL,0),
\]
and moves the object at $p_{i+1}$ along the three-segment path
\[
(iL,0)\to (iL,-h)\to ((i-1)L,-h)\to ((i-1)L,0).
\]
All other objects remain fixed.

These two moving objects travel on different horizontal rows, so their trajectories remain disjoint. Since the active rows are $y=h$ and $y=-h$, and since the other anchors lie on $y=0$ and are spaced by $L>2d$, choosing $h$ and the inactive anchors as above ensures that at every moment any two distinct object centers are more than $2d$ apart. By subdividing each segment into a constant number of sufficiently short steps, we obtain a constant-length video in which every object moves by at most $d$ between consecutive frames. Hence, each gadget satisfies the continuity condition and realizes exactly the adjacent transposition $\tau_i$ on the first five objects, while acting as the identity on objects $6,\dots,k$.

Now let
\[
w=\tau_{a_1}\tau_{a_2}\cdots\tau_{a_m}
\]
be an input word over $\{\tau_1,\tau_2,\tau_3,\tau_4\}$. We use the standard right-to-left composition convention, so the group element represented by $w$ is
\[
[w]=\tau_{a_m}\circ\cdots\circ\tau_{a_1},
\]
that is, $\tau_{a_1}$ acts first, then $\tau_{a_2}$, and so on. We map $w$ to a video $V(w)$ by concatenating the gadgets
\[
G_{a_1},G_{a_2},\dots,G_{a_m},
\]
starting from the frame in which the $k$ objects occupy the anchors $p_1,\dots,p_k$.

Each gadget ends with every object exactly at an anchor position, so the next gadget starts from the correct configuration. Therefore, the global permutation induced by the full video is exactly
\[
[w]=\tau_{a_m}\circ\cdots\circ\tau_{a_1}
\]
on the first five objects, and the identity on the remaining objects. Hence, the global permutation $\pi$ is the identity on all $k$ objects if and only if the word $w$ evaluates to the identity in $S_5$.

Finally, the reduction is a constant-block substitution: each input symbol $\tau_i$ is replaced by the constant-length gadget $G_i$. Hence, the output video is computable from the input word by a uniform $\mathbf{AC}^0$ reduction.
\end{proof}

% \begin{theorem}
% For any fixed $k \ge 5$, the problem $\mathrm{TRACK}_k$ is $\mathbf{NC}^1$-complete.
% \end{theorem}
% \begin{proof}
% By Lemma \ref{lem:1}, $\mathrm{TRACK}_k \in \mathbf{NC}^1$. By Lemma \ref{lem:2}, $\mathrm{WORD}_{S_5} \le_{\mathbf{AC}^0} \mathrm{TRACK}_k$ for every fixed $k\ge 5$. Since $\mathrm{WORD}_{S_5}$ is $\mathbf{NC}^1$-complete, it follows that $\mathrm{TRACK}_k$ is $\mathbf{NC}^1$-hard. Therefore, $\mathrm{TRACK}_k$ is $\mathbf{NC}^1$-complete.
% \end{proof}

\completeness*
\begin{proof}
Lemma~\ref{lem:1} gives $\mathrm{TRACK}_k \in \mathbf{NC}^1$, and Lemma~\ref{lem:2} gives $\mathbf{NC}^1$-hardness for every fixed $k \ge 5$ via an $\mathbf{AC}^0$ reduction from $\mathrm{WORD}_{S_5}$, which is $\mathbf{NC}^1$-complete. Therefore, $\mathrm{TRACK}_k$ is $\mathbf{NC}^1$-complete.
\end{proof}

% \completeness*
% \begin{proof}
% ... your proof here ...
% \end{proof}

% video 001
\clearpage
\section{Quantitative Analysis}
\label{sec:quantitative_analysis}
\begin{figure}[H]
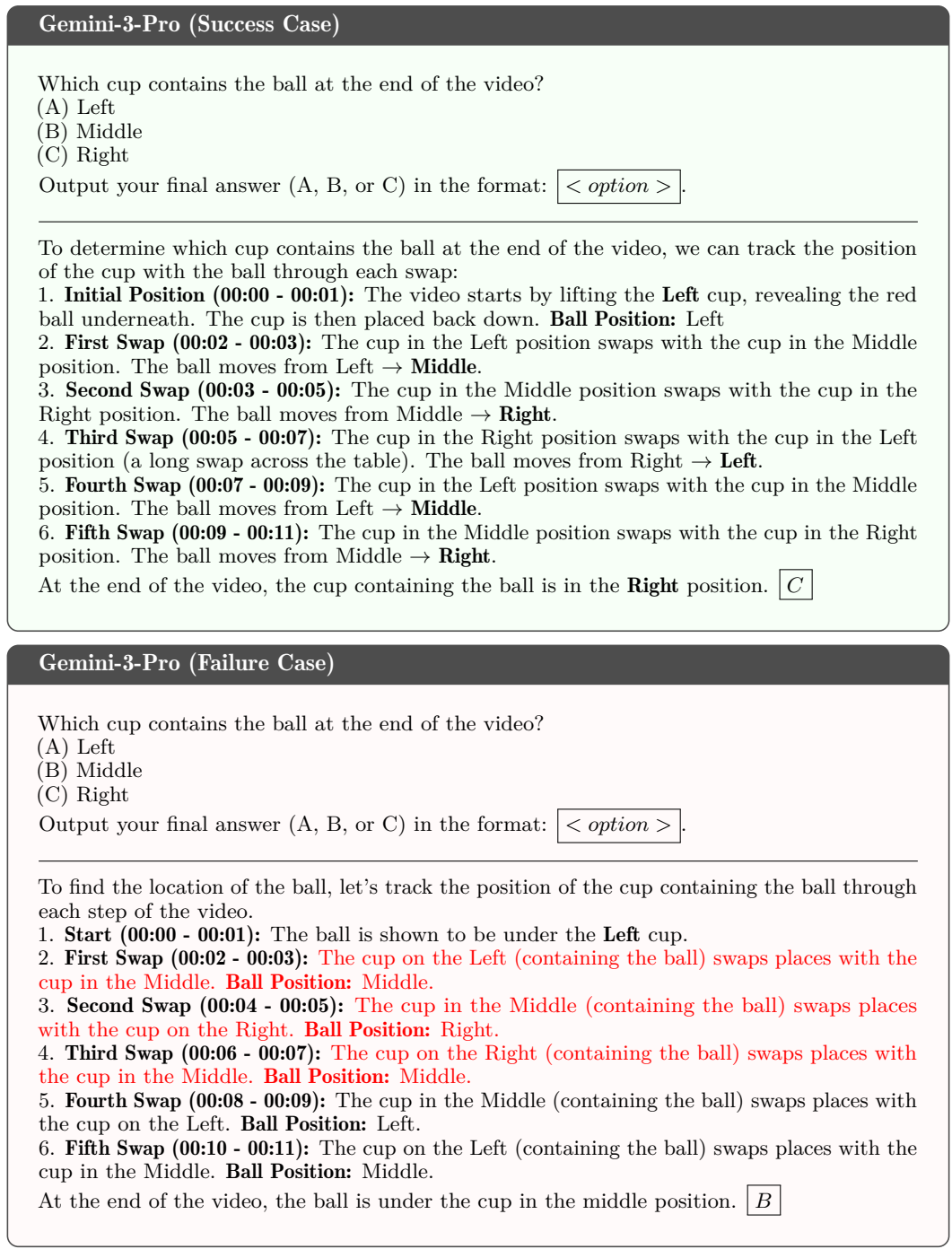

\centering
\begin{tcolorbox}[
    colback=green!3, 
    colframe=black!70, 
    arc=1.5mm, 
    boxrule=0.6pt, 
    left=10pt, right=10pt, top=10pt, bottom=10pt,
    fonttitle=\bfseries,
    title=Gemini-3-Pro (Success Case)
]
    \small
    Which cup contains the ball at the end of the video? \\
    (A) Left \\
    (B) Middle \\
    (C) Right \\
    Output your final answer (A, B, or C) in the format: $\boxed{<option>}$.
    \vspace{0.8em}
    \hrule
    \vspace{0.8em}

To determine which cup contains the ball at the end of the video, we can track the position of the cup with the ball through each swap:

1. \textbf{Initial Position (00:00 - 00:01):} The video starts by lifting the \textbf{Left} cup, revealing the red ball underneath. The cup is then placed back down. \textbf{Ball Position:} Left

2. \textbf{First Swap (00:02 - 00:03):} The cup in the Left position swaps with the cup in the Middle position. The ball moves from Left $\rightarrow$ \textbf{Middle}.

3. \textbf{Second Swap (00:03 - 00:05):} The cup in the Middle position swaps with the cup in the Right position. The ball moves from Middle $\rightarrow$ \textbf{Right}.

4. \textbf{Third Swap (00:05 - 00:07):} The cup in the Right position swaps with the cup in the Left position (a long swap across the table). The ball moves from Right $\rightarrow$ \textbf{Left}.

5. \textbf{Fourth Swap (00:07 - 00:09):} The cup in the Left position swaps with the cup in the Middle position. The ball moves from Left $\rightarrow$ \textbf{Middle}.

6. \textbf{Fifth Swap (00:09 - 00:11):} The cup in the Middle position swaps with the cup in the Right position. The ball moves from Middle $\rightarrow$ \textbf{Right}.

At the end of the video, the cup containing the ball is in the \textbf{Right} position. $\boxed{C}$
\end{tcolorbox}

\begin{tcolorbox}[
    colback=red!2, % Slight red tint to denote a "Wrong" answer/failure case
    colframe=black!70, 
    arc=1.5mm, 
    boxrule=0.6pt, 
    left=10pt, right=10pt, top=10pt, bottom=10pt,
    fonttitle=\bfseries,
    title= Gemini-3-Pro (Failure Case)
]
    \small
    Which cup contains the ball at the end of the video? \\
    (A) Left \\
    (B) Middle \\
    (C) Right \\
    Output your final answer (A, B, or C) in the format: $\boxed{<option>}$.
    \vspace{0.8em}
    \hrule
    \vspace{0.8em}
To find the location of the ball, let's track the position of the cup containing the ball through each step of the video.

% video 007
1. \textbf{Start (00:00 - 00:01):} The ball is shown to be under the \textbf{Left} cup.

2. \textbf{First Swap (00:02 - 00:03):} \textcolor{red}{The cup on the Left (containing the ball) swaps places with the cup in the Middle. \textbf{Ball Position:} Middle.}

3. \textbf{Second Swap (00:04 - 00:05):} \textcolor{red}{The cup in the Middle (containing the ball) swaps places with the cup on the Right. \textbf{Ball Position:} Right.}

4. \textbf{Third Swap (00:06 - 00:07):} \textcolor{red}{The cup on the Right (containing the ball) swaps places with the cup in the Middle.  \textbf{Ball Position:} Middle.}

5. \textbf{Fourth Swap (00:08 - 00:09):} The cup in the Middle (containing the ball) swaps places with the cup on the Left. \textbf{Ball Position:} Left.

6. \textbf{Fifth Swap (00:10 - 00:11):} The cup on the Left (containing the ball) swaps places with the cup in the Middle.  \textbf{Ball Position:} Middle.

At the end of the video, the ball is under the cup in the middle position. $\boxed{B}$
\end{tcolorbox}
\caption{Examples of Gemini-3-Pro responses on the VET-Bench cups-game task. (Top) A successful case where the model correctly tracks all the intermediate states. (Bottom) A failure case where perception errors (highlighted in \textcolor{red}{red}) lead to an incorrect final prediction.}
\label{fig:gemini-3-pro-cup}
\end{figure}

% card_021, card_40
\begin{figure}[H]
\centering
\begin{tcolorbox}[
    colback=green!3, 
    colframe=black!70, 
    arc=1.5mm, 
    boxrule=0.6pt, 
    left=10pt, right=10pt, top=10pt, bottom=10pt,
    fonttitle=\bfseries,
    title=Gemini-3-Pro (Success Case)
]
    \footnotesize
    Where is the Queen of Hearts at the end of the video? \\
    (A) Left \\
    (B) Middle \\
    (C) Right \\
    Output your final answer (A, B, or C) in the format: $\boxed{<option>}$.
    \vspace{0.8em}
    \hrule
    \vspace{0.8em}
To determine the final position of the Queen of Hearts, let's track its movement step-by-step through the video:

1. \textbf{Start (00:00):} The Queen of Hearts is revealed in the \textbf{Middle}. It then flips face down.\\
2. \textbf{First Move (00:01 - 00:03):} The card in the Middle (Queen) and the card on the Left swap places. The Middle card moves over to the Left position, and the Left card moves under to the Middle position. Current Position of Queen: \textbf{Left}.\\
3. \textbf{Second Move (00:03 - 00:05):} The card now in the Middle (which was the empty card) and the card on the Left (Queen) swap places again. The Middle card moves over to the Left, and the Left card moves under to the Middle. Current Position of Queen: \textbf{Middle}.\\
4. \textbf{Third Move (00:05 - 00:07):} The card on the Left (empty) and the card in the Middle (Queen) swap. The Left card moves over to the Middle, and the Middle card moves under to the Left. Current Position of Queen: \textbf{Left}.\\
5. \textbf{Fourth Move (00:07 - 00:09):} The card on the Left (Queen) and the card in the Middle (empty) swap. The Left card moves over to the Middle, and the Middle card moves under to the Left. Current Position of Queen: \textbf{Middle}.\\
6. \textbf{Fifth Move (00:09 - 00:11):} The card in the Middle (Queen) and the card on the Right (empty) swap. The Right card moves over to the Middle, and the Middle card moves under to the Right. Current Position of Queen: \textbf{Right}.

At the end of the video, the Queen of Hearts is in the Right position. $\boxed{C}$
\end{tcolorbox}

\begin{tcolorbox}[
    colback=red!2, % Slight red tint to denote a "Wrong" answer/failure case
    colframe=black!70, 
    arc=1.5mm, 
    boxrule=0.6pt, 
    left=10pt, right=10pt, top=10pt, bottom=10pt,
    fonttitle=\bfseries,
    title= Gemini-3-Pro (Failure Case)
]
    \small
    Where is the Queen of Hearts at the end of the video? \\
    (A) Left \\
    (B) Middle \\
    (C) Right \\
    Output your final answer (A, B, or C) in the format: $\boxed{<option>}$.
    \vspace{0.8em}
    \hrule
    \vspace{0.8em}
Based on the visual evidence provided in the frames, here is the step-by-step tracking of the Queen of Hearts:

1. \textbf{Start (00:00 - 00:01):} The Queen of Hearts is the face-up card on the \textbf{Right}. It flips face-down at 00:01. \textbf{Current Position:} Right

2. \textbf{Swap 1 (00:01 - 00:03):} The first movement involves the Left and Middle cards. The Left card moves high/right, and the Middle card moves low/left. The Right card (Queen) remains stationary. \textbf{Current Position:} Right

3.  \textbf{Swap 2 (00:03 - 00:05):} The second movement involves the Middle and Right cards. The Right card (Queen) moves high/left to the middle position. The Middle card moves low/right to the right position. \textbf{Current Position:} Middle

4.  \textbf{Swap 3 (00:05 - 00:07):} The third movement again involves the Middle and Right cards. The Middle card (Queen) moves low/right back to the right position. The Right card moves high/left to the middle position. \textbf{Current Position:} Right

5.  \textbf{Swap 4 (00:07 - 00:09):} \textcolor{red}{The fourth movement involves the Left and Middle cards. The Right card (Queen) remains stationary. \textbf{Current Position:} Right}

6.  \textbf{Swap 5 (00:09 - 00:11):} The final movement involves the Left and Middle cards again. The Right card (Queen) remains stationary. \textbf{Final Position:} Right

Therefore, at the end of the video, the Queen of Hearts is in the Right position. $\boxed{C}$
\end{tcolorbox}
\caption{Examples of Gemini-3-Pro responses on the VET-Bench cards-game task. (Top) A successful case where the model correctly tracks all the intermediate states. (Bottom) A failure case where perception errors (highlighted in \textcolor{red}{red}) lead to an incorrect final prediction.}
\label{fig:gemini-3-pro-card}
\end{figure}

% coarse description
% qwen cup 28
\begin{figure}[H]
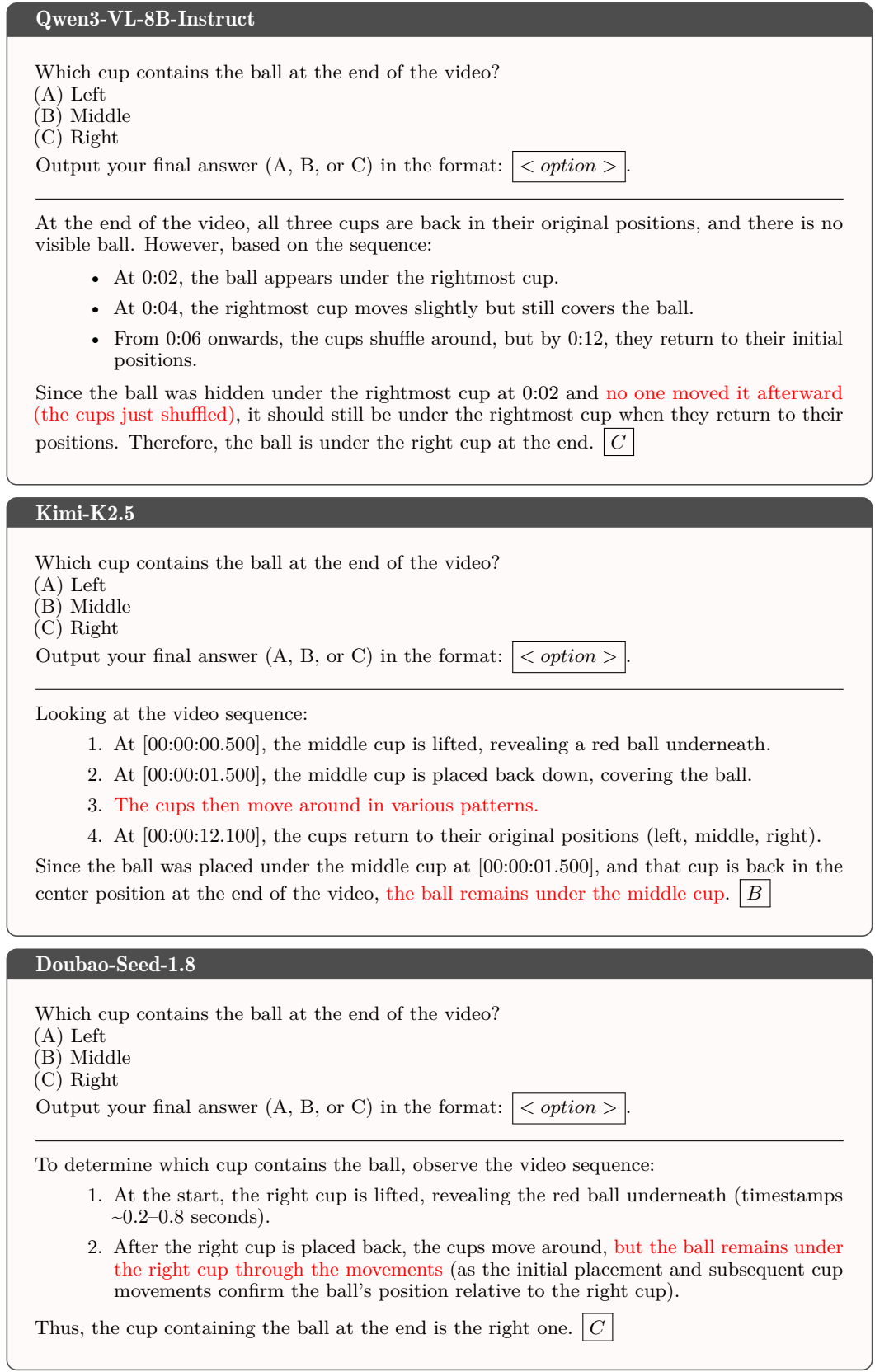

\centering
\begin{tcolorbox}[
    colback=red!2, 
    colframe=black!70, 
    arc=1.5mm, 
    boxrule=0.6pt, 
    left=10pt, right=10pt, top=10pt, bottom=10pt,
    fonttitle=\bfseries,
    title= Qwen3-VL-8B-Instruct
]
    \footnotesize
    Which cup contains the ball at the end of the video? \\
    (A) Left \\
    (B) Middle \\
    (C) Right \\
    Output your final answer (A, B, or C) in the format: $\boxed{<option>}$.
    \vspace{0.8em}
    \hrule
    \vspace{0.8em}
At the end of the video, all three cups are back in their original positions, and there is no visible ball. However, based on the sequence:

\begin{itemize}
    \item At 0:02, the ball appears under the rightmost cup.
    \item At 0:04, the rightmost cup moves slightly but still covers the ball.
    \item From 0:06 onwards, the cups shuffle around, but by 0:12, they return to their initial positions.
\end{itemize}

Since the ball was hidden under the rightmost cup at 0:02 and \textcolor{red}{no one moved it afterward (the cups just shuffled)}, it should still be under the rightmost cup when they return to their positions. Therefore, the ball is under the right cup at the end. $\boxed{C}$
\end{tcolorbox}

\begin{tcolorbox}[
    colback=red!2, % Slight red tint to denote a "Wrong" answer/failure case
    colframe=black!70, 
    arc=1.5mm, 
    boxrule=0.6pt, 
    left=10pt, right=10pt, top=10pt, bottom=10pt,
    fonttitle=\bfseries,
    title= Kimi-K2.5
]
    \small
    Which cup contains the ball at the end of the video? \\
    (A) Left \\
    (B) Middle \\
    (C) Right \\
    Output your final answer (A, B, or C) in the format: $\boxed{<option>}$.
    \vspace{0.8em}
    \hrule
    \vspace{0.8em}
Looking at the video sequence:
\begin{enumerate}
    \item At [00:00:00.500], the middle cup is lifted, revealing a red ball underneath.
    \item At [00:00:01.500], the middle cup is placed back down, covering the ball.
    \item \textcolor{red}{The cups then move around in various patterns.}
    \item At [00:00:12.100], the cups return to their original positions (left, middle, right).
\end{enumerate}
Since the ball was placed under the middle cup at [00:00:01.500], and that cup is back in the center position at the end of the video, \textcolor{red}{the ball remains under the middle cup}. $\boxed{B}$
\end{tcolorbox}

\begin{tcolorbox}[
    colback=red!2, % Slight red tint to denote a "Wrong" answer/failure case
    colframe=black!70, 
    arc=1.5mm, 
    boxrule=0.6pt, 
    left=10pt, right=10pt, top=10pt, bottom=10pt,
    fonttitle=\bfseries,
    title= Doubao-Seed-1.8
]
    \small
    Which cup contains the ball at the end of the video? \\
    (A) Left \\
    (B) Middle \\
    (C) Right \\
    Output your final answer (A, B, or C) in the format: $\boxed{<option>}$.
    \vspace{0.8em}
    \hrule
    \vspace{0.8em}
To determine which cup contains the ball, observe the video sequence:
\begin{enumerate}
    \item At the start, the right cup is lifted, revealing the red ball underneath (timestamps \textasciitilde 0.2–0.8 seconds).
    \item After the right cup is placed back, the cups move around, \textcolor{red}{but the ball remains under the right cup through the movements} (as the initial placement and subsequent cup movements confirm the ball’s position relative to the right cup).
\end{enumerate}
Thus, the cup containing the ball at the end is the right one. $\boxed{C}$
\end{tcolorbox}

% \begin{tcolorbox}[
%     colback=red!2, % Slight red tint to denote a "Wrong" answer/failure case
%     colframe=black!70, 
%     arc=1.5mm, 
%     boxrule=0.6pt, 
%     left=10pt, right=10pt, top=10pt, bottom=10pt,
%     fonttitle=\bfseries,
%     title= Molmo2
% ]
%     \small
%     Which cup contains the ball at the end of the video? \\
%     (A) Left \\
%     (B) Middle \\
%     (C) Right \\
%     Output your final answer (A, B, or C) in the format: $\boxed{<option>}$.
%     \vspace{0.8em}
%     \hrule
%     \vspace{0.8em}
%     \textcolor{red}{(C) Right}
% \end{tcolorbox}

\caption{Examples of model responses on VET-Bench. Errors are highlighted in \textcolor{red}{red}.}
\label{fig:model_performances}
\end{figure}

% Hallucination
\section{Hallucination}
\begin{figure}[H]
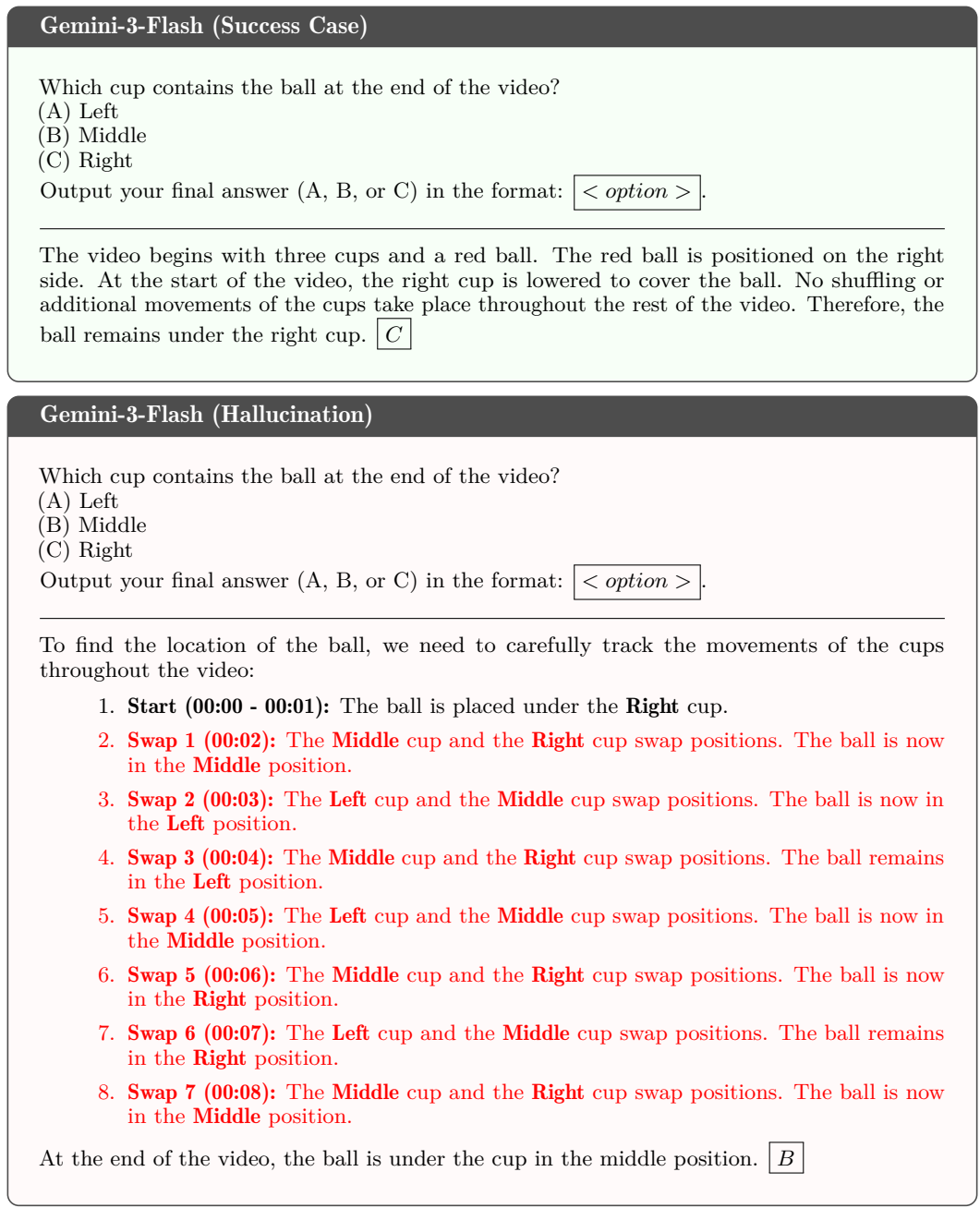

\centering
\begin{tcolorbox}[
    colback=green!3, 
    colframe=black!70, 
    arc=1.5mm, 
    boxrule=0.6pt, 
    left=10pt, right=10pt, top=10pt, bottom=10pt,
    fonttitle=\bfseries,
    title=Gemini-3-Flash (Success Case)
]
    \footnotesize
    Which cup contains the ball at the end of the video? \\
    (A) Left \\
    (B) Middle \\
    (C) Right \\
    Output your final answer (A, B, or C) in the format: $\boxed{<option>}$.
    \vspace{0.8em}
    \hrule
    \vspace{0.8em}
The video begins with three cups and a red ball. The red ball is positioned on the right side. At the start of the video, the right cup is lowered to cover the ball. No shuffling or additional movements of the cups take place throughout the rest of the video. Therefore, the ball remains under the right cup. $\boxed{C}$
\end{tcolorbox}

\begin{tcolorbox}[
    colback=red!2, % Slight red tint to denote a "Wrong" answer/failure case
    colframe=black!70, 
    arc=1.5mm, 
    boxrule=0.6pt, 
    left=10pt, right=10pt, top=10pt, bottom=10pt,
    fonttitle=\bfseries,
    title= Gemini-3-Flash (Hallucination)
]
    \small
    Which cup contains the ball at the end of the video? \\
    (A) Left \\
    (B) Middle \\
    (C) Right \\
    Output your final answer (A, B, or C) in the format: $\boxed{<option>}$.
    \vspace{0.8em}
    \hrule
    \vspace{0.8em}
To find the location of the ball, we need to carefully track the movements of the cups throughout the video:

\begin{enumerate}
    \item \textbf{Start (00:00 - 00:01):} The ball is placed under the \textbf{Right} cup.
    \textcolor{red}{
    \item \textbf{Swap 1 (00:02):} The \textbf{Middle} cup and the \textbf{Right} cup swap positions. The ball is now in the \textbf{Middle} position.
    \item \textbf{Swap 2 (00:03):} The \textbf{Left} cup and the \textbf{Middle} cup swap positions. The ball is now in the \textbf{Left} position.
    \item \textbf{Swap 3 (00:04):} The \textbf{Middle} cup and the \textbf{Right} cup swap positions. The ball remains in the \textbf{Left} position.
    \item \textbf{Swap 4 (00:05):} The \textbf{Left} cup and the \textbf{Middle} cup swap positions. The ball is now in the \textbf{Middle} position.
    \item \textbf{Swap 5 (00:06):} The \textbf{Middle} cup and the \textbf{Right} cup swap positions. The ball is now in the \textbf{Right} position.
    \item \textbf{Swap 6 (00:07):} The \textbf{Left} cup and the \textbf{Middle} cup swap positions. The ball remains in the \textbf{Right} position.
    \item \textbf{Swap 7 (00:08):} The \textbf{Middle} cup and the \textbf{Right} cup swap positions. The ball is now in the \textbf{Middle} position.
    }
\end{enumerate}

At the end of the video, the ball is under the cup in the middle position. $\boxed{B}$
\end{tcolorbox}
\caption{Examples of hallucination by Gemini-3-Flash in zero-swap cases. Gemini-3-Pro and Gemini-3-Flash often correctly identify the initial state but then hallucinate a long sequence of non-existent swaps (as shown in \textcolor{red}{red}), causing the final prediction to be incorrect.}
\label{fig:gemini-3-flash-hallucination}
\end{figure}

% \begin{figure}[H]
% \centering
% \begin{tcolorbox}[
%     colback=gray!5, 
%     colframe=gray!50, 
%     arc=1mm, 
%     boxrule=0.5pt, 
%     left=8pt, right=8pt, top=8pt, bottom=8pt,
%     fonttitle=\bfseries,
%     title=Qualitative Example: Temporal Object Tracking
% ]
%     \small
%     \textbf{Question:} Which cup is the Red ball under at the end of the video? \\
%     A. Left \quad B. Middle \quad C. Right \\\\
%     Output your final answer (A, B, C, etc.) in LaTeX format: $\boxed{<option>}$

%     \vspace{0.5em}
%     \hrule
%     \vspace{0.8em}

%     \textbf{Model (Doubao-Seed-2.0-Mini):} \\
%     $\boxed{\text{B}}$
% \end{tcolorbox}
% \caption{Example of direct answer by Doubao-Seed-2.0-Mini.}
% \label{fig:model_reasoning_trace}
% \end{figure}

% \section{Training}
% Training is performed using QLoRA.
% We use a combination of prompting and aligning method.
% In order to best preserve the model’s original tracking ability, the trajectories are designed to be \emph{in-distribution} with object trajectories in VET-Bench. 

\section{Molmo2-SGCoT}

\begin{figure}[H]
\centering
\begin{tcolorbox}[
    colback=red!2, % Slight red tint to denote a "Wrong" answer/failure case
    colframe=black!70, 
    arc=1.5mm, 
    boxrule=0.6pt, 
    left=10pt, right=10pt, top=10pt, bottom=10pt,
    fonttitle=\bfseries,
    title= Molmo2
]
    \small
    Which cup contains the ball at the end of the video? \\
    (A) Left \\
    (B) Middle \\
    (C) Right \\
    Output your final answer (A, B, or C) in the format: $\boxed{<option>}$.
    \vspace{0.8em}
    \hrule
    \vspace{0.8em}
    \textcolor{red}{(C) Right}
\end{tcolorbox}

\begin{tcolorbox}[
    colback=red!2, % Slight red tint to denote a "Wrong" answer/failure case
    colframe=black!70, 
    arc=1.5mm, 
    boxrule=0.6pt, 
    left=10pt, right=10pt, top=10pt, bottom=10pt,
    fonttitle=\bfseries,
    % title= Molmo2
]
    \small
    Where is the Queen of Hearts at the end of the video? \\
    (A) Left \\
    (B) Middle \\
    (C) Right \\
    Output your final answer (A, B, or C) in the format: $\boxed{<option>}$.
    \vspace{0.8em}
    \hrule
    \vspace{0.8em}
    \textcolor{red}{(A) Left}
\end{tcolorbox}
\caption{Examples of Molmo2's prompts and responses on VET-Bench. Molmo2 responds in a direct-answer mode without explicitly performing spatiotemporal tracking, resulting in near-random performance.}
\label{fig:molmo2_direct_answer}
\end{figure}

\begin{figure}[H]
\centering
\begin{tcolorbox}[
    colback=green!3, 
    colframe=black!70, 
    arc=1mm, 
    boxrule=0.5pt, 
    left=8pt, right=8pt, top=8pt, bottom=8pt,
    fonttitle=\bfseries,
    title=Molmo2-SGCoT
]
    \small
    Track the cup that contains the ball and answer which cup contains the ball at the end of the video?
    \vspace{0.5em}
    \hrule
    \vspace{0.8em}
\begin{Verbatim}[fontsize=\small, breaklines=true, breakanywhere=true, breaksymbol={}]
<tracks coords="0.0 1 216 517;0.5 1 204 307;1.0 1 204 307;1.5 1 208 410;2.0 1 219 554;2.5 1 265 672;3.0 1 369 702;3.5 1 474 624;4.0 1 502 528;4.5 1 504 528;5.0 1 505 528;5.5 1 505 528;6.0 1 505 528;6.5 1 505 528;7.0 1 505 528;7.5 1 505 528;8.0 1 502 502;8.5 1 447 402;9.0 1 349 385;9.5 1 258 435;10.0 1 217 546;10.5 1 262 672;11.0 1 369 706;11.5 1 475 624;12.0 1 505 528">the cup that contains the ball</tracks> Answer: middle.
    \end{Verbatim}
\end{tcolorbox}

\begin{tcolorbox}[
    colback=green!3, 
    colframe=black!70, 
    arc=1mm, 
    boxrule=0.5pt, 
    left=8pt, right=8pt, top=8pt, bottom=8pt,
    fonttitle=\bfseries,
    % title=Molmo2-SGCoT
]
    \small
    Track the Queen of Hearts and answer where is the Queen of Hearts at the end of the video?
    \vspace{0.5em}
    \hrule
    \vspace{0.8em}
\begin{Verbatim}[fontsize=\small, breaklines=true, breakanywhere=true, breaksymbol={}]
<tracks coords="0.0 1 745 512; 0.5 1 745 512; 1.0 1 745 500; 1.5 1 742 425; 2.0 1 681 296; 2.5 1 589 283; 3.0 1 514 383; 3.5 1 506 496; 4.0 1 492 498; 4.5 1 502 504; 5.0 1 503 498; 5.5 1 503 498; 6.0 1 508 517; 6.5 1 492 504; 7.0 1 503 498; 7.5 1 503 492; 8.0 1 462 329; 8.5 1 364 275; 9.0 1 281 350; 9.5 1 247 500; 10.0 1 247 500; 10.5 1 247 500; 11.0 1 247 500; 11.5 1 247 500; 12.0 1 247 500">the Queen of Hearts</tracks> Answer: left.
    \end{Verbatim}
\end{tcolorbox}
\caption{Examples of Molmo2-SGCoT's prompts and responses on VET-Bench. Top: cups game; Bottom: cards game.}
\label{fig:example_of_Molmo2-SGCoT}
\end{figure}

\begin{figure}[H]
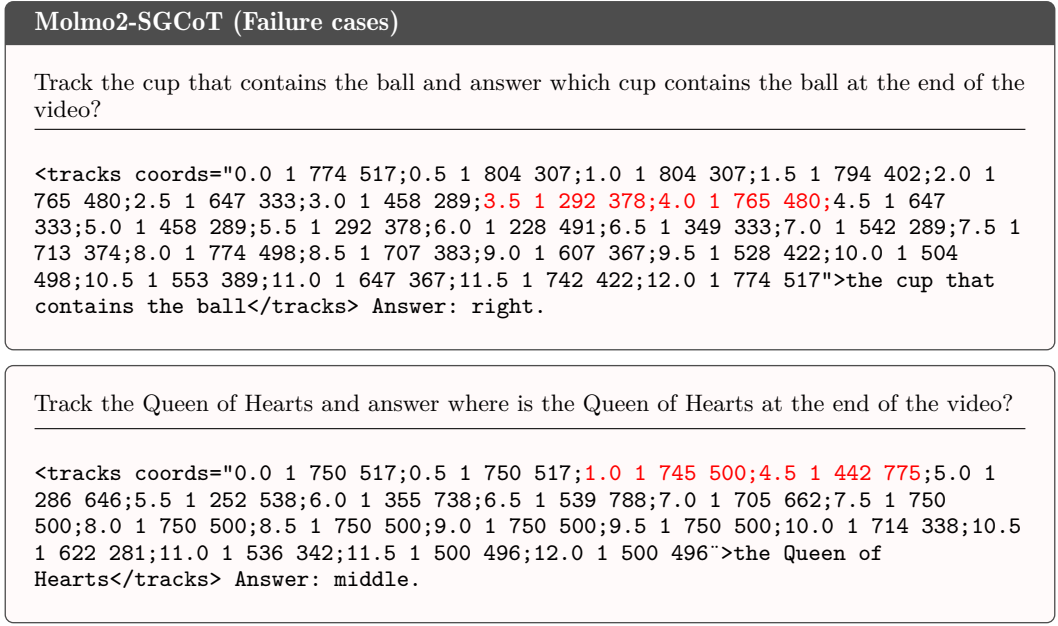

\centering
\begin{tcolorbox}[
    colback=red!2, 
    colframe=black!70, 
    arc=1mm, 
    boxrule=0.5pt, 
    left=8pt, right=8pt, top=8pt, bottom=8pt,
    fonttitle=\bfseries,
    title=Molmo2-SGCoT (Failure cases)
]
    \small
    Track the cup that contains the ball and answer which cup contains the ball at the end of the video?
    \vspace{0.5em}
    \hrule
    \vspace{0.8em}
\begin{Verbatim}[
    fontsize=\small, 
    breaklines=true, 
    breakanywhere=true, 
    breaksymbol={},
    % This is the key: it defines \ { } as the command escape and delimiters
    commandchars=\\\{\} 
]
<tracks coords="0.0 1 774 517;0.5 1 804 307;1.0 1 804 307;1.5 1 794 402;2.0 1 765 480;2.5 1 647 333;3.0 1 458 289;\textcolor{red}{3.5 1 292 378;4.0 1 765 480;}4.5 1 647 333;5.0 1 458 289;5.5 1 292 378;6.0 1 228 491;6.5 1 349 333;7.0 1 542 289;7.5 1 713 374;8.0 1 774 498;8.5 1 707 383;9.0 1 607 367;9.5 1 528 422;10.0 1 504 498;10.5 1 553 389;11.0 1 647 367;11.5 1 742 422;12.0 1 774 517">the cup that contains the ball</tracks> Answer: right.
\end{Verbatim}
\end{tcolorbox}

\begin{tcolorbox}[
    colback=red!2, 
    colframe=black!70, 
    arc=1mm, 
    boxrule=0.5pt, 
    left=8pt, right=8pt, top=8pt, bottom=8pt,
    fonttitle=\bfseries,
    % title=Example of Molmo2-SGCoT's wrong prompt and response on Cup-game
]
    \small
    Track the Queen of Hearts and answer where is the Queen of Hearts at the end of the video?
    \vspace{0.5em}
    \hrule
    \vspace{0.8em}
\begin{Verbatim}[
    fontsize=\small, 
    breaklines=true, 
    breakanywhere=true, 
    breaksymbol={},
    % This is the key: it defines \ { } as the command escape and delimiters
    commandchars=\\\{\} 
]
<tracks coords="0.0 1 750 517;0.5 1 750 517;\textcolor{red}{1.0 1 745 500;4.5 1 442 775};5.0 1 286 646;5.5 1 252 538;6.0 1 355 738;6.5 1 539 788;7.0 1 705 662;7.5 1 750 500;8.0 1 750 500;8.5 1 750 500;9.0 1 750 500;9.5 1 750 500;10.0 1 714 338;10.5 1 622 281;11.0 1 536 342;11.5 1 500 496;12.0 1 500 496\">the Queen of Hearts</tracks> Answer: middle.
\end{Verbatim}
\end{tcolorbox}

\caption{Examples of Molmo2-SGCoT failure cases. Errors typically arise in SGCoT when the model fails to differentiate between visually identical objects. This manifests as abrupt ``jumps'' in the spatial (top) or temporal (bottom) domains (highlighted in \textcolor{red}{red}), leading to incorrect terminal locations and subsequent failures in the final answer.}
\label{fig:sgcot_wrong}
\end{figure}

\section{Training}
\label{app:training}
\begin{table}[H]
\centering
\small
\begin{tabular}{lc}
\toprule
% \textbf{Hyperparameter} & \textbf{Value} \\
% \midrule
Base Model & Molmo2-8B \\
Training Method & QLoRA \\
Gradient Accumulation Steps & 4 \\
Batch Size & 64 \\
Learning Rate & 1e-4 \\
Training Samples & 300 \\
Number of Epochs & 1 \\
\midrule
\multicolumn{2}{c}{\textbf{LoRA Configuration}} \\
\midrule
LoRA Rank ($r$) & 16 \\
LoRA Alpha ($\alpha$) & 16 \\
Target Modules & \texttt{att\_proj, attn\_out, ff\_proj, ff\_out} \\
\bottomrule
\end{tabular}
\caption{Molmo2-SGCoT Training Details. We apply QLoRA to fine-tune the language model while keeping the vision encoder weights frozen.}
\label{tab:training_hyperparameters}
\end{table}

\end{document}